\documentclass[11pt]{article}
\usepackage[margin = 1in]{geometry}

\usepackage[round]{natbib}
\usepackage{xcolor}
\definecolor{plum}  {rgb}{.4,0,.4}
\definecolor{brickred} {rgb}{0.6,0,0}
\usepackage[plainpages=false,pdfpagelabels,colorlinks=true,linkcolor=brickred,citecolor=brickred]{hyperref}

\usepackage[utf8]{inputenc}
\usepackage[T1]{fontenc}
\usepackage{hyperref}
\usepackage{url}
\usepackage{booktabs}
\usepackage{amsfonts}
\usepackage{nicefrac}
\usepackage{microtype,marginnote}

\usepackage{mathtools}
\usepackage{amsmath}
\usepackage{graphicx}
\usepackage{subfigure}

\usepackage{caption}
\usepackage{amsfonts,psfrag,epsfig}
\usepackage{color}
\usepackage{amsmath,amssymb,amsthm}
\usepackage{afterpage}

\makeatletter
\def\thmheadbrackets#1#2#3{%
  \thmname{#1}\thmnumber{\@ifnotempty{#1}{ }\@upn{#2}}%
  \thmnote{ {\the\thm@notefont[#3]}}}
\makeatother

\newtheoremstyle{brackets}
  {}
  {}
  {\itshape}
  {}
  {\bfseries}
  {.}
  { }
  {\thmheadbrackets{#1}{#2}{#3}}

\theoremstyle{brackets}

\newtheorem{theorem}{Theorem}

\newtheorem{lemma}[theorem]{Lemma}

\newtheorem{definition}{Definition}

\usepackage{color}

\definecolor{Gray}{gray}{0.9}

\usepackage{colortbl}
\usepackage[flushleft]{threeparttable}
\usepackage{booktabs,caption}
\usepackage{multirow}
\newcommand{\reals}{\mathbb{R}}

\newcommand{\relu}{\textsc{ReLU}}
\newcommand{\step}{\textsc{Step}}
\newcommand{\minwidth}{w_{\min}}
\newcommand{\mc}{\mathcal}
\newcommand{\defeq}{:=}
\newcommand{\norm}[1]{\left\|{#1}\right\|}

\newcommand{\linf}[1]{\norm{#1}_\infty}

\begin{document} 
\title{Minimum Width for Universal Approximation}

\author{Sejun Park\thanks{School of Electrical Engineering, KAIST}
\and Chulhee Yun\thanks{Laboratory for Information and Decision Systems, MIT}
\and
Jaeho Lee$^*$
\and Jinwoo Shin$^*$\thanks{Graduate School of AI, KAIST\newline $~$\quad~~\,E-mails: \texttt{sejun.park@kaist.ac.kr, chulheey@mit.edu, jaeho-lee@kaist.ac.kr, jinwoos@kaist.ac.kr}.}
}

\maketitle
\begin{abstract}
The universal approximation property of width-bounded networks has been studied as a dual of classical universal approximation results on depth-bounded networks.
However, the critical width enabling the universal approximation has not been exactly characterized in terms of the input dimension $d_x$ and the output dimension $d_y$.
In this work, we provide the first definitive result in this direction for networks using the $\relu$ activation functions: The minimum width required for the universal approximation of the $L^p$ functions is exactly $\max\{d_x+1,d_y\}$.
We also prove that the same conclusion does not hold for the uniform approximation with $\relu$, but does hold with an additional threshold activation function.
Our proof technique can be also used to derive a tighter upper bound on the minimum width required for the universal approximation using networks with general activation functions.
\end{abstract}

\section{Introduction}
\label{sec:intro}
The study of the expressive power of neural networks investigates what class of functions neural networks can/cannot represent or approximate. 
Classical results in this field are mostly focused on shallow neural networks. An example of such results is the universal approximation theorem \citep{cybenko89,hornik89,pinkus99}, which shows that a neural network with fixed depth and arbitrary width can approximate any continuous function on a compact set, up to arbitrary accuracy, if the activation function is continuous and nonpolynomial.
Another line of research studies the memory capacity of neural networks \citep{baum88,huang98,huang03}, trying to characterize the maximum number of data points that a given neural network can memorize.

After the advent of deep learning, researchers started to investigate the benefit of depth in the expressive power of neural networks, in an attempt to understand the success of deep neural networks. 
This has led to interesting results showing the existence of functions that require the network to be extremely wide for shallow networks to approximate, while being easily approximated by deep and narrow networks \citep{telgarsky16, eldan16,lin17,poggio17}. A similar trade-off between depth and width in expressive power is also observed in the study of the memory capacity of neural networks \citep{yun19, vershynin20}.

In search of a deeper understanding of the depth in neural networks, a \emph{dual} scenario of the classical universal approximation theorem has also been studied \citep{lu17, hanin17,johnson19,kidger19}.
Instead of bounded depth and arbitrary width studied in classical results, the dual problem studies whether universal approximation is possible with a network of \emph{bounded width and arbitrary depth}. 
A very interesting characteristic of this setting is that there exists a \emph{critical threshold} on the width that allows a neural network to be a universal approximator. For example, one of the first results \citep{lu17} in the literature shows that universal approximation of $L^1$ functions from $\reals^{d_x}$ to $\reals$ is possible for a width-$(d_x+4)$ $\relu$ network, but \emph{impossible} for a width-$d_x$ $\relu$ network. This implies that the minimum width required for universal approximation lies between $d_x+1$ and $d_x+4$.
Subsequent results have shown upper/lower bounds on the minimum width, but none of the results has succeeded in a tight characterization of the minimum width.

\subsection{What is known so far?}
\label{sec:relatedwork}
\begin{table}
\caption{A summary of known upper/lower bounds on minimum width for universal approximation. In the table, $\mathcal K \subset \reals^{d_x}$ denotes a compact domain, and $p \in [1, \infty)$. ``Conti.'' is short for continuous.}
\vspace{8pt}
\centering
\begin{threeparttable}
\begin{tabular}{ |l|c c|c|  }
\hline 
\multicolumn{1}{|c|}{\textbf{Reference}} & \textbf{Function class} & \textbf{Activation $\rho$} & \textbf{Upper\,/\,lower bounds}\\ 
\hline \hline
\multirow{2}{120pt}[-1pt]{\citet{lu17}} & $L^1(\reals^{d_x}, \reals)$ & $\relu$ & $d_x + 1 \leq \minwidth \leq d_x + 4$\\
& $L^1(\mathcal K, \reals)$ & $\relu$ & $\minwidth \geq d_x$\\
\hline
\citet{hanin17} & $C(\mathcal K, \reals^{d_y})$ & $\relu$ & $d_x + 1 \leq \minwidth \leq d_x + d_y$\\
\hline
\citet{johnson19} & $C(\mathcal K, \reals)$ & uniformly conti.\tnote{$\dag$} & $\minwidth \geq d_x + 1$\\
\hline
\multirow{3}{120pt}[-1pt]{\citet{kidger19}} & $C(\mathcal K, \reals^{d_y})$ & conti.\ nonpoly\tnote{$\ddag$} &  $\minwidth \leq d_x + d_y + 1$\\
& $C(\mathcal K, \reals^{d_y})$ & nonaffine poly & $\minwidth \leq d_x + d_y + 2$\\
& $L^p(\reals^{d_x}, \reals^{d_y})$ & $\relu$ & $\minwidth \leq d_x + d_y + 1$\\
\hline \hline
\cellcolor{Gray}{\bf Ours} (Theorem~\ref{thm:relulp}) & \cellcolor{Gray}$L^p(\reals^{d_x}, \reals^{d_y})$ &
\cellcolor{Gray}$\relu$ & 
\cellcolor{Gray}$\minwidth = \max \{d_x+1, d_y\}$\\
\cellcolor{Gray}{\bf Ours} (Theorem~\ref{thm:reluuniflower}) & \cellcolor{Gray}$C([0,1], \reals^{2})$ &
\cellcolor{Gray}$\relu$ & 
\cellcolor{Gray}$\minwidth = 3 > \max \{d_x+1, d_y\}$\\
\cellcolor{Gray}{\bf Ours} (Theorem~\ref{thm:reluthresunif}) & \cellcolor{Gray}$C(\mathcal K, \reals^{d_y})$ &
\cellcolor{Gray}$\relu$+$\step$ & 
\cellcolor{Gray}$\minwidth = \max \{d_x+1, d_y\}$\\
\cellcolor{Gray}{\bf Ours} (Theorem~\ref{thm:generallp}) & \cellcolor{Gray}$L^p(\mathcal K, \reals^{d_y})$ &
\cellcolor{Gray}conti.\ nonpoly\tnote{$\ddag$} & 
\cellcolor{Gray}$\minwidth \leq \max \{d_x+2, d_y+1\}$\\
\hline
\end{tabular}
\begin{tablenotes}
\item[$\dag$] requires that $\rho$ is uniformly approximated by a sequence of one-to-one functions.
\item[$\ddag$] requires that $\rho$ is continuously differentiable at at least one point (say $z$), with $\rho'(z) \neq 0$.
\end{tablenotes}
\end{threeparttable}
\label{tbl:summary}
\end{table}

Before summarizing existing results, we first define function classes studied in the literature.
For a domain $\mc X \subseteq \reals^{d_x}$ and a codomain $\mc Y \subseteq \reals^{d_y}$, we define $C(\mc X, \mc Y)$ to be the class of continuous functions from $\mc X$ to $\mc Y$, endowed with the uniform norm: $\linf{f} \defeq \sup_{x \in \mc X} \linf{f(x)}$.
For $p \in [1,\infty)$, we also define $L^p(\mc X, \mc Y)$ to be the class of $L^p$ functions from $\mc X$ to $\mc Y$, endowed with the $L^p$-norm: $\norm{f}_p \defeq (\int_{\mc X} \norm{f(x)}_p^p dx)^{1/p}$. 
The summary of known upper and lower bounds in the literature, as well as our own results, is presented in Table~\ref{sec:intro}. We use $\minwidth$ to denote the minimum width for universal approximation.

\paragraph{First progress.} 
As aforementioned, \citet{lu17} show that universal approximation of $L^1(\reals^{d_x}, \reals)$ is possible for a width-($d_x+4$) $\relu$ network, but impossible for a width-$d_x$ $\relu$ network. These results translate into bounds on the minimum width: $d_x + 1 \leq \minwidth \leq d_x + 4$.
\citet{hanin17} consider approximation of $C(\mc K, \reals^{d_y})$, where $\mc K \subset \reals^{d_x}$ is compact. They prove that $\relu$ networks of width $d_x+d_y$ are dense in $C(\mc K, \reals^{d_y})$, while width-$d_x$ $\relu$ networks are \emph{not}. Although this result fully characterizes $\minwidth$ in case of $d_y=1$, it fails to do so for $d_y > 1$.

\paragraph{General activations.} 
Later, extensions to activation functions other than $\relu$ have appeared in the literature. \citet{johnson19} shows that if the activation function $\rho$ is uniformly continuous and can be uniformly approximated by a sequence of one-to-one functions, a width-$d_x$ network cannot universally approximate $C(\mc K, \reals)$.
\cite{kidger19} show that if $\rho$ is continuous, nonpolynomial, and continuously differentiable at at least one point (say $z$) with $\rho'(z) \neq 0$, then networks of width $d_x+d_y+1$ with activation $\rho$ are dense in $C(\mc K, \reals^{d_y})$.
Furthermore, \cite{kidger19} prove that $\relu$ networks of width $d_x+d_y+1$ are dense in $L^p(\reals^{d_x},\reals^{d_y})$.

\paragraph{Limitations of prior arts.}
Note that none of the existing works succeeds in closing the gap between the upper bound (at least $d_x+d_y$) and the lower bound (at most $d_x+1$). 
This gap is significant especially for applications with high-dimensional codomains (i.e., large $d_y$) such as image generation \citep{kingma13,goodfellow14}, language modeling \citep{devlin19,liu19}, and molecule generation \citep{gomez18,jin18}.
In the prior arts, the main bottleneck for proving an upper bound below $d_x+d_y$ is that they maintain all $d_x$ neurons to store the input and all $d_y$ neurons to construct the function output; this means every layer already requires at least $d_x+d_y$ neurons.
In addition, the proof techniques for the lower bounds only consider the input dimension $d_x$ regardless of the output dimension $d_y$.

\subsection{Summary of results}
We mainly focus on characterizing the minimum width of \textsc{ReLU} networks for universal approximation. 
Nevertheless, our results are not restricted to \textsc{ReLU} networks; they can be generalized to networks with general activation functions.
Our contributions can be summarized as follows.

\begin{list}{{\tiny$\bullet$}}{\leftmargin=1.8em}
  \setlength{\itemsep}{1pt}
  \vspace*{-4pt}
    \item 
    Theorem~\ref{thm:relulp} states that the minimum width for \textsc{ReLU} networks to be dense in $L^p(\mathbb R^{d_x},\mathbb R^{d_y})$ is exactly $\max\{d_x+1,d_y\}$.
    This is the first result fully characterizing the minimum width of \textsc{ReLU} networks for universal approximation.
    In particular, the upper bound on the minimum width is significantly smaller than the best known result $d_x+d_y+1$ \citep{kidger19}.
    
    \item Given the full characterization of $\minwidth$ of $\relu$ networks for approximating $L^p(\reals^{d_x}, \reals^{d_y})$, a~natural question arises: Is $\minwidth$ also the same for $C(\mc K, \reals^{d_y})$? We prove that it is \emph{not} the case;
    Theorem~\ref{thm:reluuniflower} shows that the minimum width for $\relu$ networks to be dense in $C([0,1],\mathbb R^{2})$ is $3$.
    Namely, \textsc{ReLU} networks of width $\max\{d_x+1,d_y\}$ are \emph{not} dense in $C(\mathcal K,\mathbb R^{d_y})$ in general.

    \item In light of Theorem~\ref{thm:reluuniflower}, is it impossible to approximate $C(\mc K, \reals^{d_y})$ in general while maintaining width $\max\{d_x+1,d_y\}$?
    Theorem~\ref{thm:reluthresunif} shows that an additional activation comes to rescue.
    We show that if networks use \emph{both} $\relu$ and threshold activation functions (which we refer to as $\step$)\footnote{The threshold activation function (i.e., $\step$) denotes $x \mapsto \mathbf{1}[x\ge0]$.}, they can universally approximate $C(\mc K, \reals^{d_y})$ with the minimum width $\max\{d_x+1,d_y\}$.
    
    \item Our proof techniques for tight upper bounds are not restricted to \textsc{ReLU} networks. In Theorem~\ref{thm:generallp}, we extend our results to general activation functions covered in \cite{kidger19}. 
\end{list}

\subsection{Organization}
We first define necessary notation in Section~\ref{sec:notation}.
In Section~\ref{sec:mainresult}, we formally state our main results and discuss their implications.
In Section~\ref{sec:upper}, we present our ``coding scheme'' for proving upper bounds on the minimum width in Theorems~\ref{thm:relulp}, \ref{thm:reluthresunif} and \ref{thm:generallp}.
In Section~\ref{sec:lower}, we prove the lower bound in Theorem~\ref{thm:reluuniflower} by explicitly constructing a counterexample.
Finally, we conclude the paper in Section~\ref{sec:conclusion}. 
We note that all formal proofs of Theorems~\ref{thm:relulp}--\ref{thm:generallp} are presented in Appendix.

\section{Problem setup and notation}\label{sec:notation}

Throughout this paper, we consider fully-connected neural networks that can be described as an alternating composition of affine transformations and activation functions. Formally, we consider the following setup: Given a set of activation functions $\Sigma$, an $L$-layer neural network $f$ of input dimension $d_x$, output dimension $d_y$, and hidden layer dimensions $d_1, \dots, d_{L-1}$\footnote{For simplicity of notation, we let $d_x=d_0$ and $d_y=d_L$.} is represented as
\begin{align}
f:=t_L\circ\sigma_{L-1}\circ\cdots\circ t_2\circ\sigma_1\circ t_1,\label{eq:neuralnet}
\end{align}
where $t_\ell:\mathbb{R}^{d_{\ell-1}}\rightarrow\mathbb{R}^{d_{\ell}}$ is an affine transformation and $\sigma_\ell$ is a vector of activation functions: 
\begin{align*}
    \sigma_\ell(x_1,\dots,x_{d_\ell})=\big(\rho_1(x_1),\dots,\rho_{d_\ell}(x_{d_\ell})\big),
\end{align*}
where $\rho_i\in\Sigma$. While we mostly consider the cases where $\Sigma$ is a singleton (e.g., $\Sigma = \{\textsc{ReLU}\}$), we also consider the case where $\Sigma$ contains both \textsc{ReLU} and \textsc{Step} activation functions as in Theorem \ref{thm:reluthresunif}. We denote a neural network with $\Sigma=\{\rho\}$ by a ``$\rho$ network'' and a neural network with $\Sigma=\{\rho_1,\rho_2\}$ by a ``$\rho_1$+$\rho_2$ network.''
We define the \emph{width} $w$ of $f$ as the maximum over $d_1,\dots,d_{L-1}$.

For describing the universal approximation of neural networks, we say $\rho$ networks (or $\rho_1$+$\rho_2$ networks) of width $w$ are dense in $C(\mathcal X,\mathcal Y)$ if for any $f^* \in C(\mathcal X,\mathcal Y)$ and $\varepsilon>0$, there exists a $\rho$ network (or a $\rho_1$+$\rho_2$ network) $f$ of width $w$ such that $\|f^*-f\|_\infty\le\varepsilon$.
Likewise, we say $\rho$ networks (or $\rho_1$+$\rho_2$ networks) are dense in $L^p(\mathcal X,\mathcal Y)$ if for any $f^* \in L^p(\mathcal X,\mathcal Y)$ and $\varepsilon>0$, there exists a $\rho$ network (or a $\rho_1$+$\rho_2$ network) $f$ such that $\|f^*-f\|_p\le\varepsilon$.

\section{Minimum width for universal approximation}\label{sec:mainresult}

\paragraph{$L^p$ approximation with $\relu$.}
We present our main theorems in this section.
First, for universal approximation of $L^p(\mathbb R^{d_x},\mathbb R^{d_y})$ using \textsc{ReLU} networks, we give the following theorem.
\begin{theorem}\label{thm:relulp}
For any $p \in [1, \infty)$, \textsc{ReLU} networks of width $w$ are dense in $L^p(\mathbb{R}^{d_x},\mathbb{R}^{d_y})$ if and only if $w\ge \max\{d_x+1,d_y\}$.
\end{theorem}
This theorem shows that the minimum width $\minwidth$ for universal approximation is exactly equal to $\max \{d_x+1, d_y\}$. In order to provide a tight characterization of $\minwidth$, we show three new upper and lower bounds: $\minwidth \leq \max\{d_x+1, d_y\}$ through a construction utilizing a coding approach, $\minwidth \geq d_y$ through a volumetric argument, and $\minwidth\ge d_x+1$ through an extension of the same lower bound for $L^1(\mathbb R^{d_x},\mathbb R^{d_y})$ \citep{lu17}.
Combining these bounds gives the tight minimum width $\minwidth = \max\{d_x + 1, d_y\}$.

Notably, using our new proof technique, we overcome the limitation of existing upper bounds that require width at least $d_x + d_y$. Our construction first encodes the $d_x$ dimensional input vectors into one-dimensional codewords, and maps the codewords to target codewords using memorization, and decodes the target codewords to $d_y$ dimensional output vectors. Since we construct the map from input to target using scalar codewords, we bypass the need to use $d_x+d_y$ hidden nodes. More details are found in Section~\ref{sec:upper}.
Proofs of the lower bounds are deferred to Appendices \ref{sec:pflb:general}, \ref{sec:pfthm:relulplb}.

\paragraph{Uniform approximation with $\relu$.} We have seen in Theorem~\ref{thm:relulp} a tight characterization $\minwidth = \max\{d_x+1,d_y\}$ for $L^p(\reals^{d_x}, \reals^{d_y})$ functions. Does the same hold for $C(\mathcal K,\mathbb R^{d_y})$, for a compact $\mc K \subset \reals^{d_x}$? Surprisingly, we show that the same conclusion does \emph{not} hold in general. Indeed, we show the following result, proving that width $\max\{d_x+1,d_y\}$ is \emph{provably insufficient} for $d_x = 1, d_y = 2$.
\begin{theorem}\label{thm:reluuniflower}
\textsc{ReLU} networks of width $w$ are dense in $C([0,1],\mathbb{R}^2)$ if and only if $w \ge 3$.
\end{theorem}
Theorem~\ref{thm:reluuniflower} translates to $\minwidth = 3$, and the upper bound $\minwidth \leq 3 = d_x + d_y$ is given by \citet{hanin17}. The key is to prove a lower bound $\minwidth \geq 3$, i.e., width $2$ is not sufficient. Recall from Section~\ref{sec:relatedwork} that all the known lower bounds are limited to showing that width $d_x$ is insufficient for universal approximation. A closer look at their proof techniques reveals that they heavily rely on the fact that the hidden layers have the same dimensions as the input space. As long as the width $w > d_x$, their arguments break because such a network maps the input space into a \emph{higher-dimensional} space.

Although only for $d_x = 1$ and $d_y = 2$, we overcome this limitation of the prior arts and show that width $w = 2 > d_x$ is \emph{insufficient} for universal approximation, by providing a counterexample. We use a novel topological argument which comes from a careful observation on the image created by $\relu$ operations. In particular, we utilize the property of $\relu$ that it projects all negative inputs to zero, without modifying any positive inputs. We believe that our proof will be of interest to readers and inspire follow-up works. Please see Section~\ref{sec:lower} for more details.

Theorem~\ref{thm:relulp} and Theorem~\ref{thm:reluuniflower} together imply that for $\relu$ networks, approximating $C(\mathcal K,\mathbb R^{d_y})$ requires more width than approximating $L^p(\mathbb R^{d_x},\mathbb R^{d_y})$.
Interestingly, this is in stark contrast with existing results, where the minimum \emph{depth} of $\relu$ networks for approximating $C(\mathcal K,\mathbb R^{d_y})$ is two \citep{leshno93} but it is  greater than two for approximating $L^p(\mathbb R^{d_x},\mathbb R^{d_y})$ \citep{qu19}.

\paragraph{Uniform approximation with $\relu$+$\step$.}
While width $\max\{d_x+1,d_y\}$ is insufficient for $\relu$ networks to be dense in $C(\mathcal K,\mathbb R^{d_y})$, an additional $\step$ activation function helps achieve the minimum width $\max\{d_x+1,d_y\}$, as stated in the theorem below.

\begin{theorem}\label{thm:reluthresunif}
\textsc{ReLU+Step} networks of width $w$ are dense in $C(\mathcal K,\mathbb{R}^{d_y})$ if and only if $w\ge \max\{d_x+1, d_y\}$.
\end{theorem}
Theorem~\ref{thm:reluuniflower} and Theorem~\ref{thm:reluthresunif} indicate that the minimum width for universal approximation is indeed \emph{dependent} on the choice of activation functions.
This is also in contrast to the classical results where $\relu$ networks of depth $2$ are universal approximators \citep{leshno93}, i.e., the minimum depths for universal approximation are identical for both \textsc{ReLU} networks and \textsc{ReLU}+\textsc{Step} networks.

Theorem~\ref{thm:reluthresunif} comes from a similar proof technique as Theorem~\ref{thm:relulp}. Due to its discontinuous nature, the $\step$ activation can be used in our encoder to quantize the input without introducing uniform norm errors. Lower bounds on $\minwidth$ can be proved in a similar way as Theorem~\ref{thm:relulp} (see Appendices \ref{sec:pflb:general}, \ref{sec:pfthm:reluthresuniflb}).

\paragraph{General activations.}
Our proof technique for upper bounds in Theorems~\ref{thm:relulp} and \ref{thm:reluthresunif} can be easily extended to networks using general activations. Indeed, we prove the following theorem, which shows that adding a width of $1$ is enough to cover the networks with general activations.

\begin{theorem}\label{thm:generallp}
Let $\rho:\mathbb{R}\rightarrow\mathbb{R}$ be any continuous nonpolynomial function which is continuously differentiable at at least one point, with nonzero derivative at that point. 
Then, $\rho$ networks of width $w$ are dense in $L^p(\mathcal K,\mathbb{R}^{d_y})$ for all $p\in[1,\infty)$ if $w\ge\max\{d_x+2,d_y+1\}$.
\end{theorem}
Please notice that unlike other theorems, Theorem~\ref{thm:generallp} only proves an upper bound $\minwidth \leq \max\{d_x+2,d_y+1\}$.
We note that Theorem~\ref{thm:generallp} significantly improves over the previous upper bound of width $d_x+d_y+1$ by \citep[Remark~4.10]{kidger19}.

\section{Tight upper bound on minimum width for universal\,approximation}\label{sec:upper}
In this section, we present the main idea for constructing networks achieving the minimum width for universal approximation, and then sketch the proofs of upper bounds in Theorems~\ref{thm:relulp}, \ref{thm:reluthresunif}, and \ref{thm:generallp}.

\subsection{Coding scheme for universal approximation}\label{sec:coding}
We now illustrate the main idea underlying the construction of neural networks that achieve the minimum width. 
To this end, we consider an approximation of a target continuous function $f^* \in C([0,1]^{d_x}, [0,1]^{d_y})$; however, our main idea can be easily generalized to other domain, codomain, and $L^p$ functions.
Our construction can be viewed as a \textit{coding scheme} in essence, consisting of three parts: \emph{encoder, memorizer, and decoder}. First, the encoder encodes an input vector to a one-dimensional codeword. Then, the memorizer maps the codeword to a one-dimensional target codeword that is encoded with respect to the corresponding target $f^*(x)$. Finally, the decoder maps the target codeword to a target vector which is sufficiently close to $f^*(x)$. Note that one can view the encoder, memorizer, and decoder as functions mapping from $d_x$-dimension to $1$-dimension, then to $1$-dimension, and finally to $d_y$-dimension.

The spirit of the coding scheme is that the three functions can be constructed using the idea of the prior results such as \citep{hanin17}.
Recall that \citet{hanin17} approximate any continuous function mapping $n$-dimensional inputs to $m$-dimensional outputs using $\relu$ networks of width $n+m$. 
Under this intuition, we construct the encoder, the memorizer, and the decoder by $\relu$+$\step$ networks (or $\relu$ networks) of width $d_x+1,2,d_y$, respectively; these constructions result in the tight upper bound $\max\{d_x+1,d_y\}$.
Here, the decoder requires width $d_y$ instead of $d_y+1$, as we only construct the first $d_y-1$ coordinates of the output, and recover the last output coordinate from a linear combination of the target codeword and the first $d_y-1$ coordinates.

Next, we describe the operation of each part. We explain their neural network constructions in subsequent subsections.

\begin{figure*}
\includegraphics[width=1\textwidth]{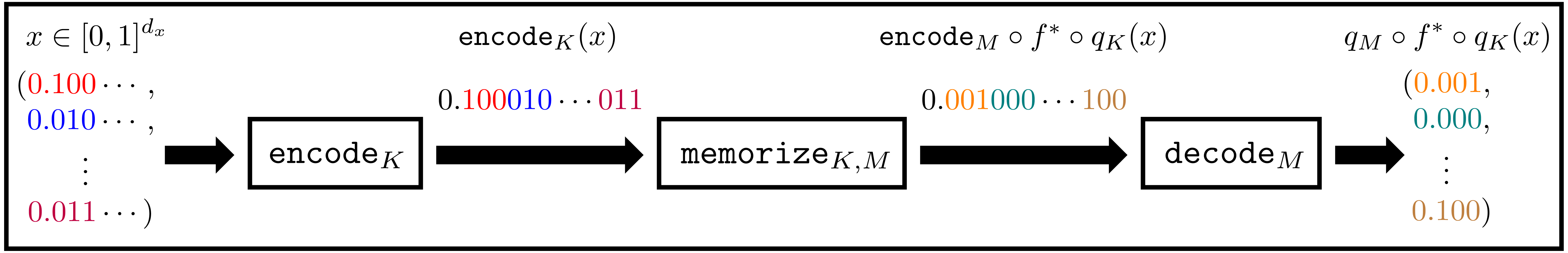}
\caption{Illustration of the coding scheme}\label{fig:coding}
\end{figure*}

\paragraph{Encoder.} 
Before introducing the encoder, we first define a quantization function $q_n:[0,1]\rightarrow\mathcal C_n$ for $n\in\mathbb N$ and $\mathcal C_n:=\{0,2^{-n},2\times2^{-n},\dots,1-2^{-n}\}$ as
\begin{align*}
q_n(x) := \max\{c\in\mathcal C_n:c\le x\}.
\end{align*}
In other words, given any $x \in [0,1)$, $q_n(x)$ preserves the first $n$ bits in the binary representation of $x$ and discards the rest; $x=1$ is mapped to $1-2^{-n}$. Note that the error from the quantization is always less than or equal to $2^{-n}$.

The encoder encodes each input $x\in[0,1]^{d_x}$ to some scalar value via the function $\mathtt{encode}_K:\mathbb{R}^{d_x}\rightarrow\mathcal C_{d_x K}$ for some $K\in\mathbb{N}$ defined as
\begin{align*}
    \mathtt{encode}_K(x):=\sum\nolimits_{i=1}^{d_x}q_K(x_i)\times2^{-(i-1)K}.
\end{align*}
In other words, $\mathtt{encode}_K(x)$ quantizes each coordinate of $x$ by a $K$-bit binary representation and concatenates the quantized coordinates into a single scalar value having a $(d_xK)$-bit binary representation. Note that if one ``decodes'' a codeword $\mathtt{encode}_K(x)$ back to a vector $\hat x$ as\footnote{Here, $\mathtt{encode}_K^{-1}$ denotes the preimage of $\mathtt{encode}_K$ and $\mathcal C_{K}^{d_x}$ is the Cartesian product of $d_x$ copies of $\mathcal C_{K}$.}
\begin{align*}
\{\hat x\}:=\big(\mathtt{encode}_K^{-1}\circ\mathtt{encode}_K(x)\big)\cap\mathcal C_{K}^{d_x}, 
\end{align*}
then $\|x-\hat x\|_\infty\le2^{-K}$.
Namely, the ``information loss'' incurred by the encoding can be made arbitrarily small by choosing large $K$.

\paragraph{Memorizer.} The memorizer maps each codeword $\mathtt{encode}_K(x)\in\mathcal C_{d_xK}$ to its target codeword
via the function $\mathtt{memorize}_{K,M}:\mathcal C_{d_xK}\rightarrow\mathcal C_{d_yM}$ for some $M\in\mathbb N$, defined as
\begin{align*}
    \mathtt{memorize}_{K,M}\big(\mathtt{encode}_K(x)\big):=\mathtt{encode}_M\big(f^*\circ q_K(x)\big)
\end{align*}
where $q_K$ is applied coordinate-wise for a vector.
We note that $\mathtt{memorizer}_{K,M}$ is well-defined as each $\mathtt{encode}_K(x)\in\mathcal C_{d_xK}$ corresponds to a unique $q_K(x)\in\mathcal C_K^{d_x}$.
Here, one can observe that the target of the memorizer contains the information of the target value 
since $\mathtt{encode}_M\big(f^*\circ q_K(x)\big)$ contains information of $f^*$ at a quantized version of $x$, and the information loss due to quantization can be made arbitrarily small by choosing large enough $K$ and $M$.

\paragraph{Decoder.} The decoder  decodes each codeword generated by the memorizer by the function $\mathtt{decode}_M:\mathcal C_{d_yM}\rightarrow\mathcal C_{M}^{d_y}$ defined as
\begin{align*}
    \mathtt{decode}_M(c):=\hat x\quad\text{where}\quad\{\hat x\}:=\mathtt{encode}_M^{-1}(c)\cap \mathcal C_M^{d_y}.
\end{align*}

Combining $\mathtt{encode}$, $\mathtt{memorize}$, and $\mathtt{decode}$ completes our coding scheme for approximating $f^*$.
One can observe that our coding scheme is equivalent to $q_M\circ f^*\circ q_K$ 
which can approximate the target function $f^*$ within any $\varepsilon>0$ error, i.e., 
\begin{align*}
\sup\nolimits_{x\in[0,1]^{d_x}}\|f^*(x)- \mathtt{decode}_M \circ \mathtt{memorize}_{K,M} \circ \mathtt{encode}_K(x)\|_\infty\le\varepsilon
\end{align*}
by choosing large enough $K,M\in\mathbb N$ so that $\omega_{f^*}(2^{-K})+2^{-M}\le\varepsilon$.\footnote{$\omega_{f^*}$ denotes the modulus of continuity of $f^*$: $\|f^*(x)-f^*(x^\prime)\|_\infty\le\omega_{f^*}(\|x-x^\prime\|_\infty)~~\forall x,x^\prime\in[0,1]^{d_x}$.}

In the remainder of this section, we discuss how each part of the coding scheme can be implemented with a neural network using $\relu$+$\step$ activations (Section~\ref{sec:pfsketch:reluthresunif}), $\relu$ activation (Section~\ref{sec:pfsketch:relulp}), and other general activations (Section~\ref{sec:pfsketchthm:generallp}).
\subsection{Tight upper~bound~on~minimum~width~of~\textsc{ReLU}+\textsc{Step}\,networks\,(Theorem\,\ref{thm:reluthresunif})}\label{sec:pfsketch:reluthresunif}
In this section, we discuss how we explicitly construct our coding scheme to approximate functions in $C(\mathcal K,\mathbb R^{d_y})$ using a width-$(\max\{d_x+1,d_y\})$ \textsc{ReLU}+\textsc{Step} network. This results in the tight upper bound in Theorem \ref{thm:reluthresunif}.

First, the encoder consists of quantization functions $q_K$ and a linear transformation.
However, as $q_K$ is discontinuous and cannot be uniformly approximated by any continuous function, we utilize the discontinuous $\step$ activation to exactly construct the encoder via a \textsc{ReLU}+\textsc{Step} network of width $d_x+1$.
On the other hand, the memorizer and the decoder maps a finite number of scalar values (i.e., $\mathcal C_{d_xK}$ and $\mathcal C_{d_yM}$, respectively) to their target values/vectors.
Such maps can be easily implemented by continuous functions (e.g., via linear interpolation), and hence, can be exactly constructed by $\relu$ networks of width $2$ and $d_y$, respectively, as discussed in Section \ref{sec:coding}.
Note that $\step$ is used only for constructing the encoder.

In summary, all parts of our coding scheme can be exactly constructed by \textsc{ReLU}+\textsc{Step} networks of width $d_x+1$, $2$, and $d_y$. Thus, the overall \textsc{ReLU}+\textsc{Step} network has width $\max\{d_x+1,d_y\}$.
Furthermore, it can approximate the target continuous function $f^*$ within arbitrary uniform error by choosing sufficiently large $K$ and $M$.
We present the formal proof in Appendix \ref{sec:pfthm:reluthresunifub}.

\subsection{Tight upper bound on minimum width of \textsc{ReLU} networks (Theorem~\ref{thm:relulp})}\label{sec:pfsketch:relulp}
The construction of width-$(\max\{d_x+1,d_y\})$ $\relu$ network for approximating $L^p(\mathbb{R}^{d_x},\mathbb R^{d_y})$ (i.e., the tight upper bound in Theorem~\ref{thm:relulp}) is almost identical to the $\relu$+$\step$ network construction in Section~\ref{sec:pfsketch:reluthresunif}.
Since any $L^p$ function can be approximated by a continuous function with compact support, we aim to approximate continuous $f^*:[0,1]^{d_x}\to[0,1]^{d_y}$ here as in our coding scheme.

Since the memorizer and the decoder can be exactly constructed by $\relu$ networks, we only discuss the encoder here.
As we discussed in the last section, the encoder cannot be uniformly approximated by continuous functions (i.e., $\relu$ networks).
Nevertheless, it can be implemented by continuous functions except for a subset of the domain around the discontinuities, and this subset can be made arbitrarily small in terms of the Lebesgue measure.
That is, we construct the encoder using a $\relu$ network of width $d_x+1$ for $[0,1]^{d_x}$ except for a small subset, which enables us to approximate the encoder in the $L^p$-norm.
Combining with the memorizer and the decoder, we obtain a $\relu$ network of width $\max\{d_x+1,d_y\}$ that approximates the target function $f^*$ in the $L^p$-norm.
We present the formal proof in Appendix \ref{sec:pfthm:relulpub}.

\subsection{Tightening upper bound on minimum width for general activations\,(Theorem\,\ref{thm:generallp})}\label{sec:pfsketchthm:generallp}
Our network construction can be generalized to general activation functions using existing results on approximation of $C(\mathcal K,\mathbb R^{d_y})$ functions. For example, \citet{kidger19} show that if the activation $\rho$ is continuous, nonpolynomial, and continuously differentiable at at least one point (say $z$) with $\rho^\prime(z) \neq 0$, then $\rho$ networks of width $d_x+d_y+1$ are dense in $C(\mathcal K,\mathbb R^{d_y})$.
Applying this result to our encoder, memorizer, and decoder constructions of $\relu$ networks, it follows that if $\rho$ satisfies the conditions above, then $\rho$ networks of width $\max \{d_x+2,d_y+1\}$ are dense in $L^p(\mathcal K,\mathbb R^{d_y})$, i.e., Theorem~\ref{thm:generallp}.
We note that any universal approximation result for $C(\mathcal{K},\mathbb{R}^{d_y})$ by networks using other activation functions, other than \citet{kidger19}, can also be combined with our construction.
We present the formal proof in Appendix \ref{sec:pfthm:generallp}.

\section{Tight lower bound on minimum width for universal\,approximation}\label{sec:lower}
The purpose of this section is to prove the tight lower bound in Theorem \ref{thm:reluuniflower}, i.e., there exist $f^*\in C([0,1],\mathbb R^2)$ and $\varepsilon>0$ satisfying the following property: For any width-2 $\relu$ network $f$, we have $    \|f^*-f\|_\infty>\varepsilon$.
Our construction of $f^*$ is based on topological properties of $\relu$ networks, which we study in Section \ref{sec:topologicalproerty}.
Then, we introduce a counterexample $f^*$ and prove that $f^*$ cannot be approximated by width-$2$ $\relu$ networks in Section \ref{sec:counter}.
\subsection{Topological properties of ReLU networks}\label{sec:topologicalproerty}
We first interpret a width-$2$ $\relu$ network $f:\mathbb{R}\to\mathbb{R}^2$ as below, following \eqref{eq:neuralnet}:
\begin{align*}
    f:=t_L\circ\sigma\circ\cdots\circ\sigma\circ t_2\circ\sigma\circ t_1
\end{align*}
where $L\in\mathbb N$ denotes the number of layers, $t_1:\mathbb{R}\rightarrow\mathbb{R}^2$ and $t_\ell:\mathbb{R}^2\to\mathbb{R}^2$ for $\ell>1$ are affine transformations, and $\sigma$ is the coordinate-wise $\relu$.
Without loss of generality, we assume that $t_\ell$ is invertible for all $\ell>1$, as invertible affine transformations are dense in the space of affine transformations on bounded support, endowed with the uniform norm.
To illustrate the topological properties of $f$ better, we reformulate $f$ as follows:
\begin{align}
    f&=(\phi_{L-1}^{-1}\circ\sigma\circ\phi_{L-1})\circ\cdots\circ(\phi_{2}^{-1}\circ\sigma\circ\phi_{2})\circ(\phi_{1}^{-1}\circ\sigma\circ\phi_{1})\circ t^\dagger\label{eq:reform}
\end{align}
where $\phi_\ell$ and $t^\dagger$ are defined as
\begin{align*}
    t^\dagger&:=t_L\circ\cdots\circ t_1\quad\text{and}\quad
    \phi_\ell:=(t_L\circ\cdots\circ t_{\ell+1})^{-1},
\end{align*}
i.e., $t_\ell=\phi_{\ell}\circ\phi_{\ell-1}^{-1}$ for $\ell \geq 2$ and $t_1=\phi_1\circ t^\dagger$.
Under the reformulation \eqref{eq:reform}, $f$ first maps inputs through an affine transformation $t^\dagger$, then it sequentially applies $\phi_\ell^{-1}\circ\sigma\circ\phi_\ell$. Here, $\phi_\ell^{-1}\circ\sigma\circ\phi_\ell$ can be viewed as changing the coordinate system using $\phi_\ell$, applying $\relu$ in the modified coordinate system, and then returning back to the original coordinate system via $\phi_\ell^{-1}$.
Under this reformulation, we present the following lemmas.
The proofs of Lemmas \ref{lem:relu}, \ref{lem:topology} are presented in Appendices \ref{sec:pflem:relu}, \ref{sec:pflem:topology}.
\begin{lemma}\label{lem:relu}
Let $\phi:\mathbb{R}^2\rightarrow\mathbb{R}^2$ be an invertible affine transformation. Then, there exist $a_1,a_2\in\mathbb{R}^2$ and $b_1,b_2\in\mathbb{R}$ such that the following statements hold for $\mathcal S:=\{x:\langle a_1,x\rangle+b_1\ge0,\langle a_2,x\rangle+b_2\ge0\}$ and $x^\prime:=\phi^{-1}\circ\sigma\circ\phi(x)$:
\begin{list}{{\tiny$\bullet$}}{\leftmargin=1.8em}
  \setlength{\itemsep}{2pt}
  \vspace*{-4pt}
    \item If $x\in\mathcal S$, then $x^\prime=x$.
    \item If $x\in\mathbb R^2\setminus\mathcal S$, then 
    $x^\prime\ne x$ and $x^\prime\in\partial\mathcal S$.\footnote{$\partial\mathcal S$ denotes the boundary set of $\mathcal S$.}
\end{list}
\end{lemma}
\begin{lemma}\label{lem:topology}
Let $\phi:\mathbb{R}^2\to\mathbb{R}^2$ be an invertible affine transformation. Suppose that $x\in\mathbb{R}^2$, $\mathcal T\subset\mathbb{R}^2$ satisfies that $x$ is in a bounded path-connected component of $\mathbb{R}^2\setminus\mathcal T$. Then, the following statements hold for $x^\prime:=\phi^{-1}\circ\sigma\circ\phi(x)$ and $\mathcal T^\prime:=\phi^{-1}\circ\sigma\circ\phi(\mathcal T)$: 
\begin{list}{{\tiny$\bullet$}}{\leftmargin=1.8em}
  \setlength{\itemsep}{2pt}
  \vspace*{-4pt}
    \item If $x^\prime=x$ and $x^\prime\notin\mathcal T^\prime$, then $x^\prime$ is in a bounded path-connected component of $\mathbb R^2\setminus \mathcal T^\prime$.
    \item If $x^\prime\ne x$, then $x^\prime\in\mathcal T^\prime$. 
\end{list}
\end{lemma}
Lemma \ref{lem:relu} follows from the fact that output of $\relu$ is identity to nonnegative coordinates, and is zero to negative coordinates. In particular, $a_1,b_1$ and $a_2,b_2$ in Lemma \ref{lem:relu} correspond to the axes of the ``modified'' coordinate system before applying $\sigma$.
Under the same property of $\relu$, Lemma \ref{lem:topology} states that if a point $x$ is surrounded by a set $\mathcal T$, after applying $\phi^{-1}\circ\sigma\circ\phi$, either the point stays at the same position and surrounded by the image of $\mathcal T$ or
intersects with the image of $\mathcal T$.
Based on these observations, we are now ready to introduce our counterexample.
\begin{figure*}[t]
\centering
\subfigure[]{
\centering
\includegraphics[width=0.25\textwidth]{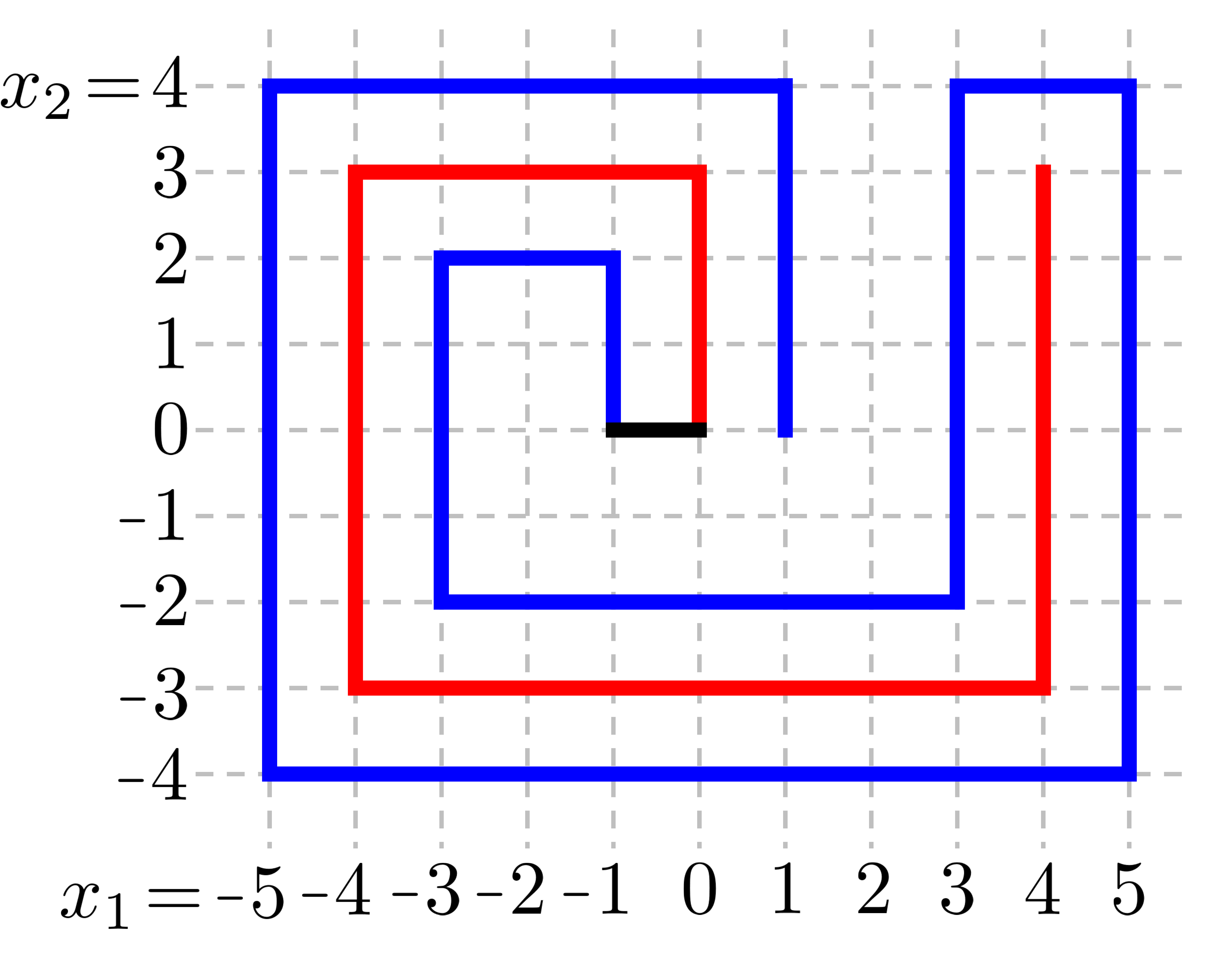}
\label{fig:counter1}
}
\subfigure[]{
\centering
\includegraphics[width=0.25\textwidth]{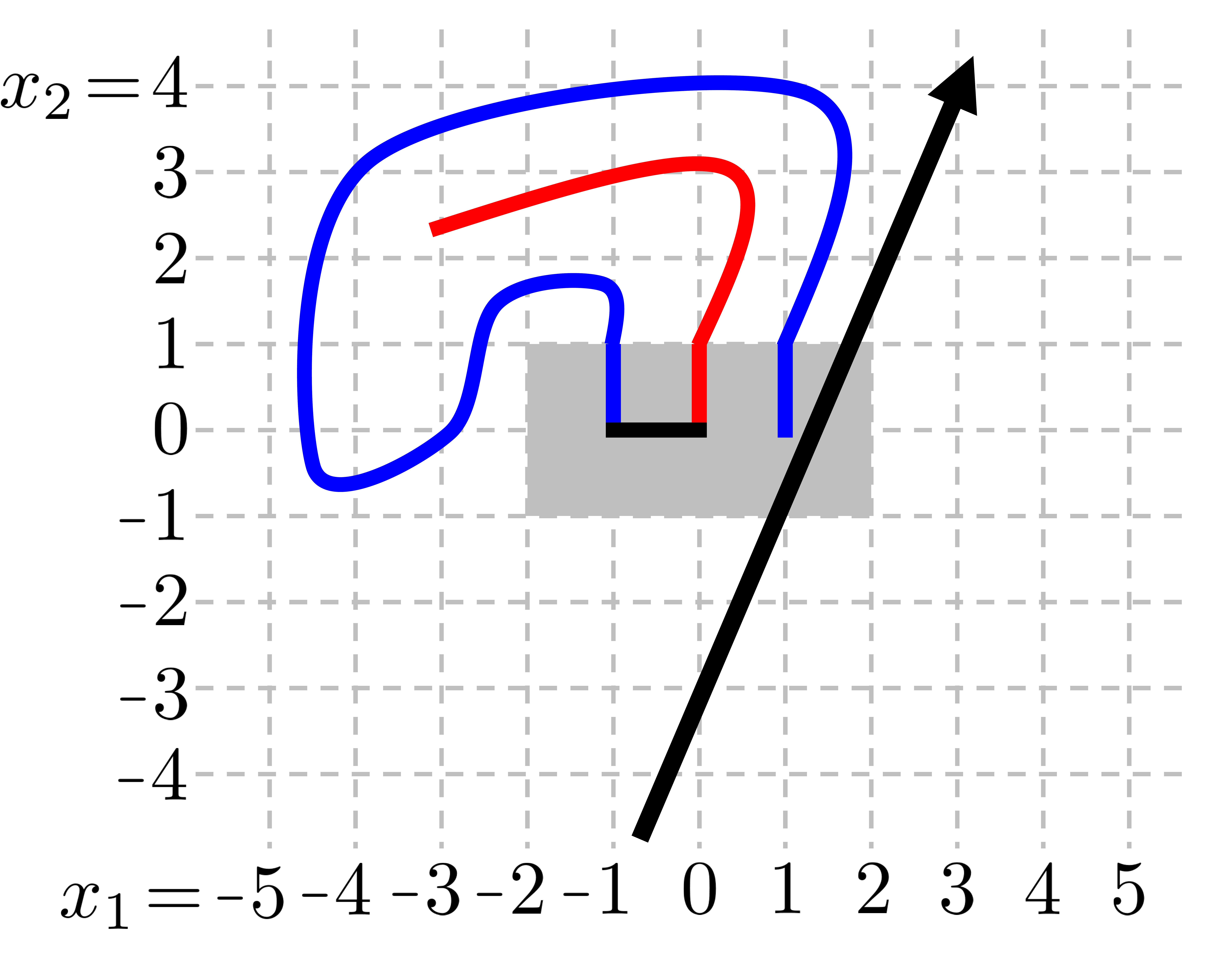}
\label{fig:counter2}
}
\subfigure[]{
\centering
\includegraphics[width=0.25\textwidth]{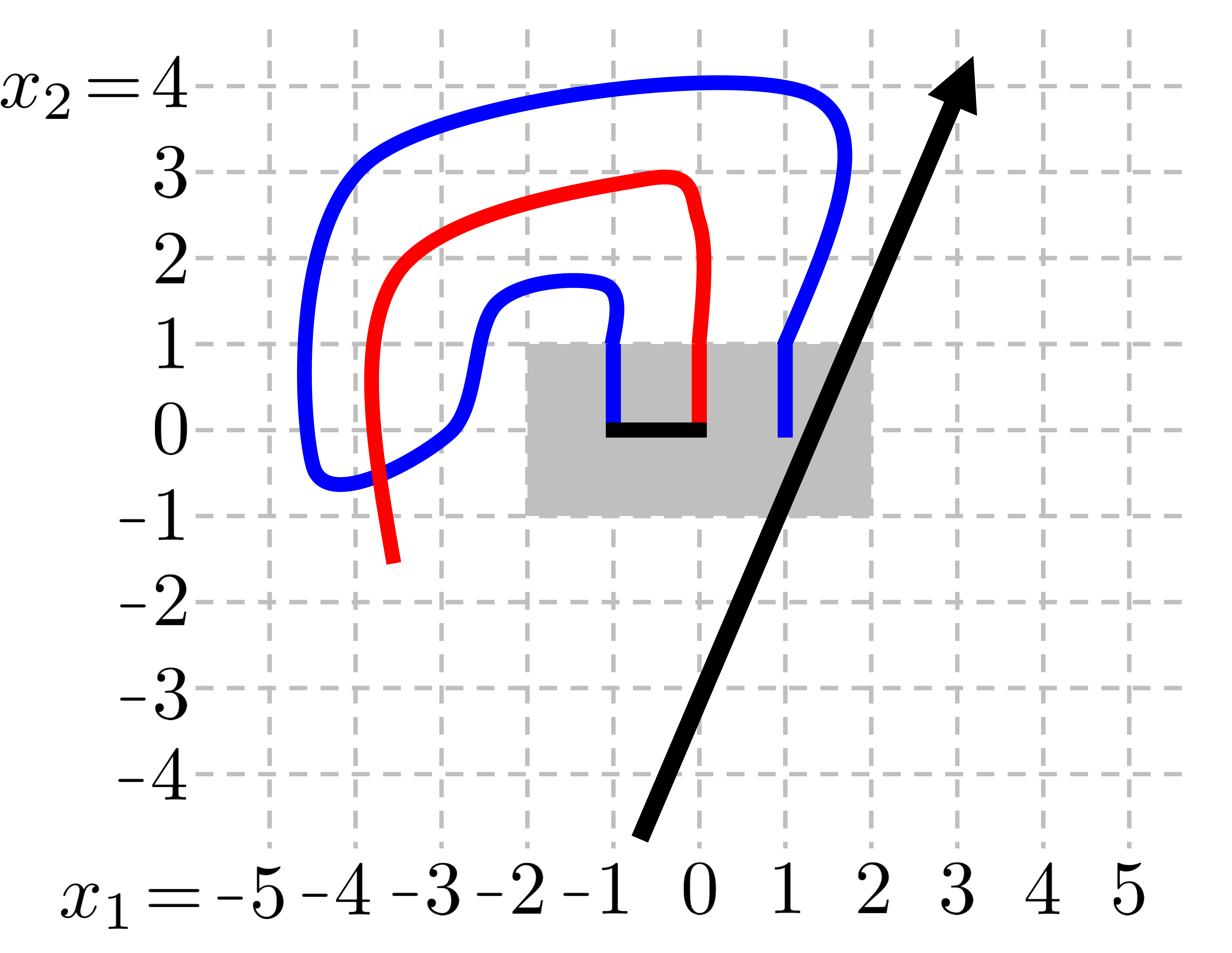}
\label{fig:counter3}
}

\caption{(a) Illustration of the image of  $f^*:[0,1]\to\mathbb{R}^2$ (b, c) Examples of $g_{\ell^*}([0,1])$.}
\label{fig:counter}
\end{figure*}
\subsection{Counterexample}\label{sec:counter}
Our counterexample $f^*:[0,1]\rightarrow\mathbb{R}^2$ is illustrated in Figure \ref{fig:counter1} where $f^*([0,p_1])$ is drawn in red from $(4,3)$ to $(0,0)$, $f^*((p_1,p_2))$ is drawn in black from $(0,0)$ to $(-1,0)$, and $f^*([p_2,1])$ is drawn in blue from $(-1,0)$ to $(1,0)$, for some $0<p_1<p_2<1$, e.g., $p_1=\frac13,p_2=\frac23$.
In this section, we suppose for contradiction that there exists a $\relu$ network $f$ of width 2 such that $\|f^*-f\|_\infty\le\frac1{100}$.
To this end, consider the mapping by the first $\ell$ layers of $f$:
\begin{align*}
    g_\ell:=(\phi_{\ell}^{-1}\circ\sigma\circ\phi_\ell)\circ\cdots\circ(\phi_1^{-1}\circ\sigma\circ\phi_1)\circ t^\dagger.
\end{align*}
Our proof is based on the fact if $g_\ell(x)=g_\ell(x^\prime)$, then $f(x)=f(x^\prime)$. Thus, the following must hold:
\begin{equation}
\label{eq:nointersect}
\text{if }
\|f^*-f\|_\infty\le\tfrac1{100},
\text{ then } 
g_\ell([0,p_1])\cap g_\ell([p_2,1])=\emptyset
\text{ for all }\ell \geq 1.
\end{equation}

Let $\mathcal B:=(-2,2)\times(-1,1)$ (the gray box in Figure~\ref{fig:counter2}) and $\ell^*\in\mathbb N$ be the largest number such that $\phi_\ell^{-1}\circ\sigma\circ\phi_\ell(\mathcal B) \neq \mathcal B$. This means that after the $\ell^*$-th layer, everything inside the box $\mathcal B$ never gets affected by $\relu$ operations. 
By the definition of $\ell^*$ and Lemma \ref{lem:relu}, there exists a line (e.g., the arrow in  Figure~\ref{fig:counter2}) intersecting with $\mathcal B$, such that the image $g_{\ell^*}([0,1])$ lies in one side of the line.
Since the image of the entire network $f([0,p_1])$ is on both sides of the line, we have $g_{\ell^*}(\left[0,p_1\right])\ne f(\left[0,p_1\right])$, which implies that the remaining layers $\ell^*+1, \dots, L-1$ must have moved the image $g_{\ell^*}([0,p_1])\setminus \mathcal B$ to $f([0,p_1])\setminus \mathcal B$; this also implies $g_{\ell^*}([0,p_1])\setminus\mathcal B \neq \emptyset$. 
A similar argument gives $g_{\ell^*}([p_2,1])\setminus\mathcal B \neq \emptyset$.

Since $\mc B$ cannot be modified after layer $\ell^*$, $f([0,1])\cap\mathcal B$ must have been constructed in the first $\ell^*$ layers.
This means that, as illustrated in Figures~\ref{fig:counter2} and \ref{fig:counter3}, the boundary $\partial \mathcal B$ intersects with $g_{\ell^*}([p_2,1])$ (the blue line) near points $(-1,1)$ and $(1,1)$, hence $\mc T \defeq g_{\ell^*}([p_2,1]) \cup \mathcal B$ forms a ``closed loop.'' 
Also, $\partial \mathcal B$ intersects with $g_{\ell^*}([0,p_1])$ near the point $(0,1)$, so there must exist a point in $g_{\ell^*}([0,p_1]) \setminus \mc B$ that is ``surrounded'' by $\mc T$.
Given these observations, we have the following lemma.
The proof of Lemma \ref{lem:topology2} is presented in Appendix \ref{sec:pflem:topology2}.
\begin{lemma}\label{lem:topology2}
The image $g_{\ell^*}([0,p_1])\setminus\mathcal B$ is contained in a bounded path-connected component of $\mathbb R^2\setminus\mc T$
unless $g_{\ell^*}([0,p_1])\cap g_{\ell^*}([p_2,1])\ne\emptyset$.
\end{lemma}
Figures~\ref{fig:counter2} and \ref{fig:counter3} illustrates the two possible cases of Lemma~\ref{lem:topology2}.
If $g_{\ell^*}([0,p_1])\cap g_{\ell^*}([p_2,1])\ne\emptyset$ (Figure~\ref{fig:counter3}), this contradicts \eqref{eq:nointersect}. Then, $g_{\ell^*}([0,p_1])\setminus\mathcal B$ must be contained in a bounded path-connected component of $\mathbb R^2\setminus\mc T$.
Recall that $g_{\ell^*}([0,p_1])\setminus\mathcal B$ has to move to $f([0,p_1])\setminus\mathcal B$ by layers $\ell^*+1, \dots, L-1$. However, by Lemma \ref{lem:topology}, if any point in $g_{\ell^*}([0,p_1])\setminus\mathcal B$ moves, then it must intersect with the image of $\mc T$.
If it intersects with the image of $g_{\ell^*}([p_2,1])$, then \eqref{eq:nointersect} is violated, hence a contradiction.
If it intersects with $\mathcal B$ at the $\ell^\dagger$-th layer for some $\ell^\dagger>\ell^*$, it violates the definition of $\ell^*$ as $\phi^{-1}_{\ell^\dagger}\circ\sigma\circ\phi_{\ell^\dagger}(\mathcal B)\ne\mathcal B$ by Lemma \ref{lem:relu}.
Hence, the approximation by $f$ is impossible in any cases.
This completes the proof of Theorem~\ref{thm:reluuniflower}.

\section{Conclusion}\label{sec:conclusion}
The universal approximation property of width-bounded networks is one of the fundamental problems in the expressive power theory of deep learning.
Prior arts attempt to characterize the minimum width sufficient for universal approximation; however, they only provide upper and lower bounds with large gaps.
In this work, we provide the first exact characterization of the minimum width of $\relu$ networks and $\relu$+$\step$ networks.
In addition, we observe interesting dependence of the minimum width on the target function classes and activation functions, in contrast to the minimum depth of classical results.
We believe that our results and analyses would contribute to a better understanding of the performance of modern deep and narrow network architectures.

\section*{Acknowledgements}
CY acknowledges financial supports from NSF CAREER Grant Number 1846088 and Korea Foundation for Advanced Studies.

\bibliography{reference}
\bibliographystyle{plainnat}

\appendix
\clearpage

\begin{center}{\bf {\LARGE Appendix}}
\end{center}
\vspace{0.3in}

In Appendix, we first provide proofs of upper bounds in Theorems \ref{thm:relulp}, \ref{thm:reluthresunif}, \ref{thm:generallp} in Appendix \ref{sec:suppleupper}.
In Appendix \ref{sec:supplelower}, we provide proofs of lower bounds in Theorem \ref{thm:relulp}, \ref{thm:reluthresunif} and proofs of Lemmas \ref{lem:relu}, \ref{lem:topology}, \ref{lem:topology2} used for proving the lower bound in Theorem \ref{thm:reluuniflower}.
Throughout Appendix, we denote the coordinate-wise $\relu$ as $\sigma$ and we denote the $i$-th coordinate of an output of a function $f(x)$ by $(f(x))_i$.
\section{Proofs of upper bounds}\label{sec:suppleupper}
In this section, we first provide proofs of upper bounds in Theorems \ref{thm:relulp}, \ref{thm:reluthresunif}, \ref{thm:generallp}.
Throughout this section, we denote the coordinate-wise $\relu$ by $\sigma$ and we denote the $i$-th coordinate of an output of a function $f(x)$ by $(f(x))_i$.
\subsection{Proof of tight upper bound in Theorem \ref{thm:reluthresunif}}\label{sec:pfthm:reluthresunifub}
In this section, we prove the tight upper bound on the minimum width in Theorem \ref{thm:reluthresunif}, i.e., width-($\max\{d_x+1,d_y\}$) $\relu$+$\step$ networks are dense in $C([0,1]^{d_x},\mathbb{R}^{d_y})$.
In particular, we prove that for any $f^*\in C([0,1]^{d_x},[0,1]^{d_y})$, for any $\varepsilon>0$, there exists a $\relu$+$\step$ network $f$ of width $\max\{d_x+1,d_y\}$ such that $\sup_{x\in[0,1]^{d_x}}\|f^*(x)-f(x)\|_\infty\le\varepsilon$.
Here, we note that the domain and the codomain can be easily generalized to arbitrary compact support and arbitrary codomain, respectively.

Our construction is based on the three-part coding scheme introduced in Section \ref{sec:coding}.
First, consider constructing a $\relu$+$\step$ network for the encoder.
From the definition of $q_K$, one can observe that the mapping is discontinuous and piece-wise constant.
Hence, the exact construction (or even the uniform approximation) of the encoder requires the use of discontinuous activation functions such as $\step$ (recall its definition $x \mapsto \mathbf{1}[x\ge0]$). 
We introduce the following lemma for the exact construction of $q_K$.
The proof of Lemma \ref{lem:encodingreluthres} is presented in Appendix \ref{sec:pflem:encodingreluthres}.

\begin{lemma}\label{lem:encodingreluthres}
For any $K\in\mathbb N$, there exists a \textsc{ReLU+Step} network $f:\mathbb{R}\rightarrow\mathbb{R}$ of width $2$ such that $f(x)=q_K(x)$ for all $x\in[0,1]$.
\end{lemma}
For constructing the encoder via a $\relu$+$\step$ network of width $d_x+1$, we apply $q_K$ to each input coordinate, by utilizing the extra width $1$ and using Lemma \ref{lem:encodingreluthres}.
Once we apply $q_K$ for all input coordinates, we apply the linear transformation $\sum\nolimits_{i=1}^{d_x}q_K(x_i)\times2^{-(i-1)K}$ to obtain the output of the encoder.

On the other hand, the memorizer only maps a finite number of scalar inputs to the corresponding scalar targets, which can be easily implemented by piece-wise linear continuous functions.
We show that the memorizer can be exactly constructed by a $\relu$ network of width 2 using the following lemma. The proof of Lemma \ref{lem:memorization} is presented in Appendix \ref{sec:pflem:memorization}.

\begin{lemma}\label{lem:memorization}
For any function $f^*:\mathbb R\rightarrow\mathbb R$, any finite set $\mathcal X\subset\mathbb R$,
and any compact interval $\mathcal{I} \subset\mathbb R$ containing $\mathcal X$, there exists a $\relu$ network $f:\mathbb R\rightarrow\mathbb R$ of width 2 such that $f(x)=f^*(x)$ for all $x\in\mathcal X$ and $f(\mathcal I)\subset\big[\min f^*(\mathcal X),\max f^*(\mathcal X)\big]$.
\end{lemma}
Likewise, the decoder maps a finite number of scalar inputs in $\mathcal C_{d_yM}$ to corresponding target vectors in $\mathcal C_M^{d_y}$. Here, each coordinate of a target vector corresponds to some consequent bits of the binary representation of the input.
Under the similar idea used for our implementation of the memorizer,
we show that the decoder can be exactly constructed by a $\relu$ network of width $d_y$ using the following lemma. The proof of Lemma \ref{lem:decodingmod} is presented in Appendix \ref{sec:pflem:decodingmod}.

\begin{lemma}\label{lem:decodingmod}
For any $d_y,M\in\mathbb{N}$, for any $\delta>0$, there exists a $\relu$ network $f:\mathbb R\rightarrow\mathbb R^2$ of width $d_y$ such that for all $c\in\mathcal C_{d_yM}$
\begin{align*}
    f(c)=\mathtt{decode}_M(c).
\end{align*}
Furthermore, it holds that $f(\mathbb R)\subset[0,1]^{d_y}$.
\end{lemma}

Finally, as the encoder, the memorizer, and the decoder can be constructed by $\relu$+$\step$ networks of width $d_x+1$, width 2, and width $d_y$, respectively, the width of the overall $\relu$+$\step$ network $f$ is $\max\{d_x+1,d_y\}$.
In addition, as mentioned in Section \ref{sec:coding}, choosing $K,M\in\mathbb{N}$ large enough so that $\omega_{f^*}(2^{-K})+2^{-M}\le\varepsilon$ ensures $\|f^*-f\|_\infty\le\varepsilon$.
This completes the proof of the tight upper bound in Theorem \ref{thm:reluthresunif}.

\subsection{Proof of tight upper bound in Theorem \ref{thm:relulp}}\label{sec:pfthm:relulpub}
In this section, we derive the upper bound in Theorem \ref{thm:relulp}.
In particular, we prove that for any $p\in[1,\infty)$, for any $f^*\in L^p(\mathbb{R}^{d_x},\mathbb{R}^{d_y})$, for any $\varepsilon>0$, there exists a $\relu$ network $f$ of width $\max\{d_x+1,d_y\}$ such that $\|f^*-f\|_p\le\varepsilon$.
To this end, we first note that since $f^*\in L^p(\mathbb{R}^{d_x},\mathbb{R}^{d_y})$, there exists a continuous function $f^\prime$ on a compact support such that 
\begin{align*}
    \|f^*-f^\prime\|_p\le\frac\varepsilon2.
\end{align*}
Namely, if we construct a $\relu$ network $f$ such that $\|f^\prime-f\|_p\le\frac{\varepsilon}2$, then it completes the proof.
Throughout this proof, we assume that the support of $f^\prime$ is a subset of $[0,1]^{d_x}$ and its codomain to be $[0,1]^{d_y}$ which can be easily generalized to arbitrary compact support and arbitrary codomain, respectively.

We approximate $f^\prime$ by a $\relu$ network using the three-part coding scheme introduced in Section \ref{sec:coding}.
We will refer to our implementations of the three parts as $\mathtt{encode}_{K}^\dagger(x)$, $\mathtt{memorize}_{K,M}^\dagger$,
and $\mathtt{decode}_M^\dagger$. That is, we will approximate $f^\prime$ by a $\relu$ network
\begin{equation*}
    f:=\mathtt{decode}_M^\dagger\circ\mathtt{memorize}_{K,M}^\dagger\circ\mathtt{encode}^\dagger_K.
\end{equation*}
However, unlike our construction of $\relu$+$\step$ networks in Section \ref{sec:pfthm:reluthresunifub}, $\step$ is not available, i.e., uniform approximation of $q_K$ is impossible.
Nevertheless, one can approximate $q_K$ with some continuous piece-wise linear function by approximating regions around discontinuities with some linear functions.
Under this idea, we introduce the following lemma. The proof of Lemma \ref{lem:encodingrelumod} is presented in Appendix \ref{sec:pflem:encodingrelumod}.
\begin{lemma}\label{lem:encodingrelumod}
For any $d_x,K\in\mathbb N$, for any $\gamma>0$, there exist a $\relu$ network $f:\mathbb R^{d_x}\rightarrow\mathbb R$ of width $d_x+1$ and $\mathcal D_\gamma\subset[0,1]^{d_x}$ such that for all $x\in[0,1]^{d_x}\setminus\mathcal D_\gamma$,
\begin{align*}
    f(x)=\mathtt{encode}_K(x),
\end{align*}
$\mu(\mathcal D_\gamma)<\gamma$, $f(\mathcal D_\gamma)\subset[0,1]$, and $f(\mathbb R^{d_x}\setminus[0,1]^{d_x})=\{1-2^{d_xK}\}$
where $\mu$ denotes the Lebesgue measure.
\end{lemma}
By Lemma \ref{lem:encodingrelumod}, there exist a $\relu$ network $\mathtt{encode}_{K}^\dagger$ of width $d_x+1$ and $\mathcal D_\gamma\subset[0,1]^{d_x}$ such that $\mu(\mathcal D_\gamma)<\gamma$,
\begin{align}
    &\mathtt{encode}_{K}^\dagger(x)=\mathtt{encode}_K(x)\quad\text{for all $x\in[0,1]^{d_x}\setminus\mathcal D_\gamma$},\notag\\
    &\mathtt{encode}_{K}^\dagger(\mathbb R^{d_x}\setminus[0,1]^{d_x})=\{1-2^{d_xK}\}.\label{eq:encodinginc}
\end{align}
We approximate the encoder by $\mathtt{encode}_K^\dagger$. 
Here, we note that inputs from $\mathcal D_\gamma$ would be mapped to arbitrary values by $\mathtt{encode}_{K}^\dagger$. Nevertheless, it is not critical to the error $\|f^\prime-f\|_p$ as $\mu(\mathcal D_\gamma)<\gamma$ can be made arbitrarily small by choosing a sufficiently small $\gamma$.

The implementation of the memorizer utilizes Lemma \ref{lem:memorization} as in Appendix \ref{sec:pfthm:reluthresunifub}.
However, as $f^\prime(x)=0$ for all $x\in\mathbb R^{d_x}\setminus[0,1]^{d_x}$, we construct a $\relu$ network $\mathtt{memorize}_{K,M}^\dagger$ of width 2 so that
\begin{align*}
\mathtt{memorize}_{K,M}^\dagger\left(\mathtt{encode}_{K,L}^\dagger(\mathbb{R}^{d_x}\setminus[0,1]^{d_x})\right)=\{0\}.
\end{align*}
To achieve this, we design the memorizer for $c\in\mathcal C_{d_xK}$ using Lemma \ref{lem:memorization} and based on \eqref{eq:encodinginc} as
\begin{align*}
\mathtt{memorize}_{K,M}^\dagger(c)=\begin{cases}0~&\text{if}~c=1-2^{d_xK}\\
\mathtt{memorize}_{K,M}(c)~&\text{otherwise}
\end{cases}.
\end{align*}
We note that such a design incurs an undesired error that a subset of $\mathcal E_K:=[1-2^{-K},1]^{d_x}$ might be mapped to zero after applying $\mathtt{memorize}_{K,M}^\dagger$.
Nevertheless, mapping $\mathcal E_K$ to zero is not critical to the error $\|f^\prime-f\|_p$ as $\mu(\mathcal E_K)<2^{-d_xK}$ can be made arbitrarily small by choosing a sufficiently large $K$.

We implement the decoder by a $\relu$ network $\mathtt{decode}_M^\dagger$ of width $d_y$ using Lemma \ref{lem:decodingmod} as in Appendix \ref{sec:pfthm:reluthresunifub}. Then, by Lemma \ref{lem:decodingmod}, it holds that $\mathtt{decode}_M^\dagger(\mathbb R)\subset[0,1]^{d_y}$, and hence, $f(\mathbb R^{d_x})\subset[0,1]^{d_y}$.

Finally, we bound the error $\|f^\prime-f\|_p$ utilizing the following inequality:
\begin{align}
    &\|f^\prime-f\|_p=\left(\int_{\mathbb R^{d_x}}\|f^\prime(x)-f(x)\|^p_pdx\right)^{\frac1p}\notag\\
    &=\left(\int_{[0,1]^{d_x}\setminus(\mathcal E_K\cup\mathcal D_\gamma)}\|f^\prime(x)-f(x)\|^p_pdx+\int_{\mathcal E_K\cup\mathcal D_\gamma}\|f^\prime(x)-f(x)\|^p_pdx\right)^{\frac1p}\notag\\
    &\le \left(d_y(\omega_{f^\prime}(2^{-K})+2^{-M})^p+(\mu(\mathcal E_K)+\mu(\mathcal D_\gamma))\times\sup_{x\in\mathcal E_K\cup\mathcal D_\gamma} \|f^\prime(x)-f(x)\|^p_p\right)^{\frac1p}\notag\\
    &<\left(d_y(\omega_{f^\prime}(2^{-K})+2^{-M})^p+(2^{-d_xK}+\gamma)\times\Big(\sup_{x\in[0,1]^{d_x}} \|f^\prime(x)\|_p+\sup_{x\in[0,1]^{d_x}}\|f(x)\|_p\Big)^p\right)^{\frac1p}\notag\\
    &\le\left(d_y(\omega_{f^\prime}(2^{-K})+2^{-M})^p+(2^{-d_xK}+\gamma)\times\Big(\sup_{x\in[0,1]^{d_x}} \|f^\prime(x)\|_p+(d_y)^{\frac1p}\Big)^p\right)^{\frac1p}.\notag
\end{align}
By choosing sufficiently large $K,M$ and sufficiently small $\gamma$, one can make the RHS smaller than $\varepsilon/2$ as $\sup_{x\in[0,1]^{d_x}} \|f^\prime(x)\|_p<\infty$.
This completes the proof of the tight upper bound in Theorem \ref{thm:relulp}.

\subsection{Proof of Theorem \ref{thm:generallp}}\label{sec:pfthm:generallp}
In this section, we prove Theorem \ref{thm:generallp} by proving the following statement:
For any $p\in[1,\infty)$, for any $f^*\in L^p(\mathcal K,\mathbb R^{d_y})$, for any $\varepsilon>0$, there exists a $\rho$ network $f$ of width $\max\{d_x+2,d_y+1\}$ such that $\|f^*-f\|_p\le\varepsilon$.
Here, there exists a continuous function $f^\prime\in C(\mathcal K,\mathbb R^{d_y})$ such that 
\begin{align*}
    \|f^*-f^\prime\|_p\le\frac\varepsilon2
\end{align*}
since $f^*\in L^p(\mathcal K,\mathbb R^{d_y})$.
Namely, if we construct a $\rho$ network $f$ such that $\|f^\prime-f\|_p\le\frac{\varepsilon}2$, it completes the proof.
Throughout the proof, we assume that the support of $f^\prime$ is a subset of $[0,1]^{d_x}$ and its codomain is $[0,1]^{d_y}$ which can be easily generalized to arbitrary compact support and arbitrary codomain, respectively.

Before describing our construction, we first introduce the following lemma.
\begin{lemma}[{\citet[Proposition~4.9]{kidger19}}]\label{lem:kidger}
Let $\rho:\mathbb{R}\rightarrow\mathbb{R}$ be any continuous nonpolynomial function which is continuously differentiable at at least one point, with nonzero derivative at that point. 
Then, for any $f^*\in C(\mathcal K,\mathbb R^{d_y})$, for any $\varepsilon>0$, there exists a $\rho$ network $f:\mathcal K\rightarrow\mathbb R^{d_x}\times\mathbb R^{d_y}$ of width $d_x+d_y+1$ such that for all $x\in\mathcal K$,
\begin{align*}
f(x) \defeq (y_1(x), y_2(x)),
\text{ where }
    \|y_1(x)-x\|_\infty\le\varepsilon\quad\text{and}\quad \|y_2(x)-f^*(x)\|_\infty\le\varepsilon.
\end{align*}
\end{lemma}
We note that Proposition 4.9 by \cite{kidger19} only ensures $\|y_2(x)-f^*(x)\|_\infty\le\varepsilon$; however, its proof provides $\|y_1(x)-x\|_\infty\le\varepsilon$ as well.

The proof of Theorem \ref{thm:generallp} also utilizes our coding scheme; here, we approximate $\relu$ network constructions $\mathtt{encode}_K^\dagger$, $\mathtt{memorize}_{K,M}^\dagger$, and $\mathtt{decode}_M^\dagger$ in Appendix~\ref{sec:pfthm:relulpub} by $\rho$ networks.
Using Lemma~\ref{lem:kidger}, for any $\varepsilon_1>0$, we approximate $\mathtt{encode}_K^\dagger$ by a $\rho$ network $\mathtt{encode}_K^\ddagger$ of width $d_x+2$ so that
\begin{align*}
    \left\|\mathtt{encode}_K^\ddagger(x)-\mathtt{encode}_K^\dagger(x)\right\|_\infty\le\varepsilon_1\quad\text{for all}\quad x\in\mathbb [0,1]^{d_x}\setminus\mathcal D_\gamma
\end{align*}
and $\mathtt{encode}_K^\ddagger([0,1]^{d_x})\subset[-\varepsilon_1,1+\varepsilon_1]$.
We note that $\mathtt{encode}_K^\ddagger([0,1]^{d_x})\subset[-\varepsilon_1,1+\varepsilon_1]$ is possible as $\mathtt{encode}_K^\dagger(\mathbb R^{d_x})\subset[0,1]$ by Lemma \ref{lem:encodingrelumod}.

Approximating the memorizer can be done in a similar manner.
Using Lemma \ref{lem:kidger}, for any compact interval $\mathcal I_2\subset\mathbb R$ containing $\mathcal C_{d_xK}$, for any $\varepsilon_2>0$, we approximate $\mathtt{memorize}_{K,M}^\dagger$ by a $\rho$ network $\mathtt{memorize}_{K,M}^\ddagger$ of width $3$ so that 
\begin{align*}
    \left\|\mathtt{memorize}_{K,M}^\ddagger(c)-\mathtt{memorize}_{K,M}^\dagger(c)\right\|_\infty\le\varepsilon_2\quad\text{for all}\quad c\in\mathcal C_{d_xK}
\end{align*}
and $\mathtt{memorize}_{K,M}^\ddagger(\mathcal I_2)\subset[-\varepsilon_2,1+\varepsilon_2]$.
We note that $\mathtt{memorize}_{K,M}^\ddagger(\mathcal I_2)\subset[-\varepsilon_2,1+\varepsilon_2]$ is possible as there exists $\mathtt{memorize}_{K,M}^\dagger$ (i.e., a $\relu$ network) such that $\mathtt{memorize}_{K,M}^\dagger(\mathcal I_2)\subset[0,1]$ by Lemma \ref{lem:memorization}.

For approximating the decoder, we introduce the following lemma. The proof of Lemma \ref{lem:decodingrho} is presented in Appendix \ref{sec:pflem:decodingrho}.

\begin{lemma}\label{lem:decodingrho}
For any $d_y,M\in\mathbb N$, for any $\varepsilon>0$, for any compact interval $\mathcal I\subset\mathbb R$ containing $[0,1]$, there exists a $\rho$ network $f:\mathbb R\rightarrow\mathbb R^{d_y}$ of width $d_y+1$ such that for all $c\in\mathcal I$, 
\begin{align*}
    \|f(c)-\mathtt{decode}_M^\dagger(c)\|_\infty\le\varepsilon.
\end{align*}
Namely, $f(\mathcal I)\subset[-\varepsilon,1+\varepsilon]^{d_y}$.
\end{lemma}
By Lemma \ref{lem:decodingrho}, for any compact interval $\mathcal I_3\subset\mathbb R$ containing $[0,1]$, for any $\varepsilon_3>0$, there exists a $\rho$ network $\mathtt{decode}_M^\ddagger$ of width $d_y+1$ such that 
\begin{align*}
\left\|\mathtt{decode}_M^\ddagger(c)-\mathtt{decode}_M^\dagger(c)\right\|_\infty\le\varepsilon_3\quad\text{for all}\quad c\in\mathcal C_{d_yM}
\end{align*}
and $\mathtt{decode}_M^\ddagger(\mathcal I_3)\in[-\varepsilon_3,1+\varepsilon_3]^{d_y}$. 

We approximate $f^\prime$ by a $\rho$ network $f$ of width $\max\{d_x+2,d_y+1\}$ defined as
\begin{equation*}
    f:=\mathtt{decode}_M^\ddagger\circ\mathtt{memorize}_{K,M}^\ddagger\circ\mathtt{encode}^\ddagger_K.
\end{equation*}
Here, for any $\eta>0$, by choosing sufficiently large $K,M$, sufficiently large $\mathcal I_2,\mathcal I_3$, and sufficiently small $\varepsilon_1,\varepsilon_2,\varepsilon_3$ so that $\omega_{f^\prime}(2^{-K})+2^{-M}\le\frac\eta2$ and $\omega_{\mathtt{decode}_{M}^\ddagger}\big(\omega_{\mathtt{memorize}_{K,M}^\ddagger}(\varepsilon_1)+\varepsilon_2\big)+\varepsilon_3\le\frac\eta2$, we have
\begin{align}
    \sup_{x\in[0,1]^{d_x}\setminus\mathcal D_\gamma}\|f^\prime(x)-f(x)\|_\infty\le\eta\quad\text{and}\quad f([0,1]^{d_x})\subset[-\tfrac\eta2,1+\tfrac\eta2]^{d_x}\label{eq:generallpapprox}
\end{align}
where $\omega_{\mathtt{memorize}_{K,M}^\ddagger}$ and $\omega_{\mathtt{decode}_{M}^\ddagger}$ are defined on $\mathcal I_2$ and $\mathcal I_3$, respectively.

Finally, we bound the error $\|f^\prime-f\|_p$ utilizing the following inequality:
\begin{align*}
    \|f^\prime&-f\|_p=\left(\int_{[0,1]^{d_x}}\|f^\prime(x)-f(x)\|^p_pdx\right)^{\frac1p}\\
    &=\left(\int_{[0,1]^{d_x}\setminus\mathcal D_\gamma}\|f^\prime(x)-f(x)\|^p_pdx+\int_{\mathcal D_\gamma}\|f^\prime(x)-f(x)\|^p_pdx\right)^{\frac1p}\\
    &\le\left(\sup_{x\in[0,1]^{d_x}\setminus\mathcal D_\gamma}\|f^\prime(x)-f(x)\|_p^p+\mu(\mathcal D_\gamma)\times\sup_{x\in\mathcal D_\gamma}\|f^\prime(x)-f(x)\|_p^p\right)^{\frac1p}\\
    &\le\left(\sup_{x\in[0,1]^{d_x}\setminus\mathcal D_\gamma}\|f^\prime(x)-f(x)\|_p^p+\gamma\times\Big(\sup_{x\in[0,1]^{d_x}}\|f^\prime(x)\|_p+\sup_{x\in[0,1]^{d_x}}\|f(x)\|_p\Big)^p\right)^{\frac1p}.
\end{align*}
By choosing sufficiently small $\varepsilon_1,\varepsilon_2,\varepsilon_3,\gamma$, sufficiently large $K,M$, and sufficiently large $\mathcal I_2,\mathcal I_3$, one can make the RHS smaller than $\varepsilon/2$ due to \eqref{eq:generallpapprox} and the fact that $\sup_{x\in[0,1]^{d_x}}\|f^\prime(x)\|_p<\infty$.
This completes the proof of Theorem \ref{thm:generallp}.

\subsection{Proof of Lemma \ref{lem:encodingreluthres}}\label{sec:pflem:encodingreluthres}
We construct $f: \reals \to \reals$ as $f(x):=f_{2^{K}}\circ\cdots\circ f_1(x)$ where each $f_\ell: \reals \to \reals$ is defined for $x\in[0,1]$ as
\begin{align*}
    f_\ell(x)&:=\begin{cases}
    (\ell-1)\times2^{-K}~&\text{if}~x\in[(\ell-1)\times2^{-K},\ell\times2^{-K})\\
    x~&\text{if}~x\notin[(\ell-1)\times2^{-K},\ell\times2^{-K})
    \end{cases}\\
    &~=g_{\ell3}\circ g_{\ell2}\circ g_{\ell1}(x)
\end{align*}
where $g_{\ell1}: \reals \to \reals^2$, $g_{\ell2}: \reals^2 \to \reals^2$, and $g_{\ell3}: \reals^2 \to \reals$ are defined as
\begin{align*}
    g_{\ell1}(x)&:=\big(\sigma(x),\sigma(x-\ell)\big)\\
    g_{\ell2}(x,z)&:=\big(\sigma(x+z),-\sigma(x-\ell+1)\big)\\
    g_{\ell3}(x,z)&:=\sigma(x+z)+\mathbf{1}[x\ge\ell].
\end{align*}
This directly implies that
$f(x)=q_K(x)$ for all $x\in[0,1]$ and completes the proof of Lemma \ref{lem:encodingreluthres}.

\subsection{Proof of Lemma \ref{lem:memorization}}\label{sec:pflem:memorization}
Let $x^{(1)},\dots,x^{(N)}$ be distinct elements of $\mathcal X$ in an increasing order, i.e., $x^{(i)}<x^{(j)}$ if $i<j$.
Let $x^{(0)}:=\min\mathcal I$ and $x^{(N+1)}:=\max\mathcal I$. 
Here, $x^{(0)}\le x^{(1)}$ and $x^{(N)}\le x^{(N+1)}$ as $\mathcal X\subset\mathcal I$.
Without loss of generality, we assume that $x^{(0)}=0$.
Consider a continuous piece-wise linear function $f^\dagger:[x^{(0)},x^{(N+1)}]\rightarrow\mathbb{R}$ of $N+1$ linear pieces defined as
\begin{align*}
    f^\dagger(x):=\begin{cases}
    f^*(x^{(1)})~&\text{if}~x\in[x^{(0)},x^{(1)})\\
    f^*(x^{(i)})+\frac{f^*(x^{(i+1)})-f^*(x^{(i)})}{x^{(i+1)}-x^{(i)}}(x-x^{(i)})~&\text{if}~x\in[x^{(i)},x^{(i+1)})~\text{for some $1\le i\le N-1$}\\
    f^*(x^{(N)})~&\text{if}~x\in[x^{(N)},x^{(N+1)}]
    \end{cases}.
\end{align*}
Now, we introduce the following lemma. 
\begin{lemma}\label{lem:pwl}
For any compact interval $\mathcal I\subset\mathbb R$, for any continuous piece-wise linear function $f^*:\mathcal I\rightarrow\mathbb{R}$ with $P$ linear pieces, there exists a $\relu$ network $f$ of width 2 such that $f^*(x)=f(x)$ for all $x\in\mathcal I$.
\end{lemma}
Then, from Lemma \ref{lem:pwl}, there exists a $\relu$ network $f$ of width 2 such that $f^\dagger(x)=f(x)$ for all $x\in\mathcal X$. Since $\mathcal X\subset[x^{(0)},x^{(N+1)}]=\mathcal I$ and $f^\dagger(\mathcal I)\subset\big[\min f^*(\mathcal X),\max f^*(\mathcal X)\big]$, this completes the proof of Lemma \ref{lem:memorization}.

\begin{proof}[Proof of Lemma \ref{lem:pwl}]
Suppose that $f^*$ is linear on intervals $[\min\mathcal I,x_1),[x_1,x_2),\dots,[x_{P-1},\max\mathcal I]$ and parametrized as
\begin{align*}
    f^*(x)=\begin{cases}
    a_1\times x+b_1~&\text{if}~x\in[\min\mathcal I,x_1)\\
    a_2\times x+b_2~&\text{if}~x\in[x_1,x_2)\\
    &\vdots\\
    a_P\times x+b_P~&\text{if}~x\in[x_{P-1},\max\mathcal I]
    \end{cases}
\end{align*}
for some $a_i,b_i\in\mathbb{R}$ satisfying $a_i\times x_i+b_i=a_{i+1}\times x_i+b_{i+1}$.
Without loss of generality, we assume that $\min\mathcal I=0$.

Now, we prove that for any $P \geq 1$, there exists a $\relu$ network $f:\mathcal I\rightarrow\mathbb{R}^2$ of width 2 such that $(f(x))_1=\sigma(x-x_{P-1})$ and $(f(x))_2=f^*(x)$. 
Then, $(f(x))_2$ is the desired $\relu$ network and completes the proof. 
We use the mathematical induction on $P$ for proving the existence of such $f$. If $P=1$, choosing $(f(x))_1=\sigma(x)$ and $(f(x))_2=a_1\times\sigma(x)+b_1$ completes the construction of $f$.
Now, consider $P>1$.
From the induction hypothesis, there exists a $\relu$ network $g$ of width 2 such that
\begin{align*}
    (g(x))_1&=\sigma(x-x_{P-2})\\
    (g(x))_2&=\begin{cases}
    a_1\times x+b_1~&\text{if}~x\in[\min\mathcal I,x_1)\\
    a_2\times x+b_2~&\text{if}~x\in[x_1,x_2)\\
    &\vdots\\
    a_{P-1}\times x+b_{P-1}~&\text{if}~x\in[x_{P-2},\max\mathcal I]
    \end{cases}.
\end{align*}
Then, the following construction of $f$ completes the proof of the mathematical induction:
\begin{align*}
    f(x)&=h_2\circ h_1\circ g(x)\\
    h_1(x,z)&=\big(\sigma(x-x_{P-1}+x_{P-2}),\sigma(z-K)+K\big)\\
    h_2(x,z)&=\big(x,z+(a_{P-1}-a_{P-2})\times x\big)
\end{align*}
where $K:=\min_i\min_{x\in\mathcal I}\{a_i\times x+b_i\}$.
This completes the proof of Lemma \ref{lem:pwl}.
\end{proof}

\subsection{Proof of Lemma \ref{lem:decodingmod}}\label{sec:pflem:decodingmod}
Before describing our proof, we first introduce the following lemma. The proof of Lemma \ref{lem:decoding2} is presented in Appendix \ref{sec:pflem:decoding2}.
\begin{lemma}\label{lem:decoding2}
For any $M\in\mathbb{N}$, for any $\delta>0$, 
there exists a $\relu$ network $f:\mathbb R\rightarrow\mathbb R^2$ of width $2$ such that for all $x\in[0,1]\setminus\mathcal D_{M,\delta}$,
\begin{align}
    f(x):=(y_1(x),y_2(x)),\quad\text{where}\quad y_1(x)=q_M(x),\quad
    y_2(x)=2^M\times(x-q_M(x))\label{eq:decoding},
\end{align}
and $\mathcal D_{M,\delta}:=\bigcup_{i=1}^{2^M-1}(i\times2^{-M}-\delta,i\times 2^{-M}).$ Furthermore, it holds that
\begin{align}
    f(\mathbb R)&\subset[0,1-2^{-M}]\times[0,1].\label{eq:decoding2}
\end{align}
\end{lemma}
In Lemma \ref{lem:decoding2}, one can observe that $\mathcal{C}_{d_yM}\subset[0,1]\setminus\mathcal D_{M,\delta}$ for $\delta<2^{-d_yM}$ , i.e., there exists a $\relu$ network $g$ of width 2 satisfying \eqref{eq:decoding} on $\mathcal C_{d_yM}$ and \eqref{eq:decoding2}.
$g$ enables us to extract the first $M$ bits of the binary representation of $c\in \mathcal{C}_{d_yM}$. 
Consider outputs of $g(c)$: $(g(c))_1$ for $c\in \mathcal{C}_{d_yM}$ is the first coordinate of $\mathtt{decode}_M(c)$ while $(g(c))_2\in\mathcal C_{(d_y-1)M}$ contains information on other coordinates of $\mathtt{decode}_M(c)$.
Now, consider further applying $g$ to $(g(c))_2$ and passing though the output $(g(c))_1$ via the identity function ($\relu$ is identity for positive inputs). Then, $\big(g\big((g(c))_2\big)\big)_1$ is the second coordinate of $\mathtt{decode}_M(c)$ while $\big(g\big((g(c))_2\big)\big)_2$ contains information on coordinates other than the first and the second ones of $\mathtt{decode}_M(c)$.
Under this observation, if we iteratively apply $g$ to the second output of the prior $g$ and pass through all first outputs of previous $g$'s, then we recover all coordinates of $\mathtt{decode}_M(c)$ within $d_y-1$ applications of $g$.
Note that both the first and the second outputs of the $(d_y-1)$-th $g$ correspond to the second last and the last coordinate of $\mathtt{decode}_M(c)$, respectively.
Our construction of $f$ is such an iterative $d_y-1$ applications of $g$ which can be implemented by a $\relu$ network of width $d_y$.
Here, \eqref{eq:decoding2} in Lemma \ref{lem:decoding2} enables us to achieve $f(\mathbb R)\subset[0,1]^{d_y}$.
This completes the proof of Lemma \ref{lem:decodingmod}.

\subsection{Proof of Lemma \ref{lem:encodingrelumod}}\label{sec:pflem:encodingrelumod}
To begin with, we introduce the following Lemma. The proof of Lemma \ref{lem:outsider} is presented in Appendix \ref{sec:pflem:outsider}.
\begin{lemma}\label{lem:outsider}
For any $d_x$, for any $\alpha\in(0,0.5)$, there exists a $\relu$ network $f:\mathbb R^{d_x}\rightarrow\mathbb R^{d_x}$ of width $d_x+1$ such that $f(x)=(1,\dots,1)$ for all $x\in\mathbb R^{d_x}\setminus[0,1]^{d_x}$, $f(x)=x$ for all $x\in[\alpha,1-\alpha]^{d_x}$, and $f(\mathbb R^{d_x})\subset[0,1]^{d_x}$.
\end{lemma}
By Lemma \ref{lem:outsider}, there exists a $\relu$ network $h_1$ of width $d_x+1$ such that $h_1(x)=(1,\dots,1)$ for all $x\in\mathbb R^{d_x}\setminus[0,1]^{d_x}$, $h_1(x)=x$ for all $x\in[\alpha,1-\alpha]^{d_x}$, and $h_1(\mathbb R^{d_x})\subset[0,1]^{d_x}$. 
Furthermore, by Lemma \ref{lem:decoding2}, for any $\delta>0$, 
there exists a $\relu$ network $g:\mathbb R\rightarrow\mathbb R$ of width $2$ such that $g(c)=q_K(c)$ for all $c\in[0,1]\setminus\mathcal D_{K,\delta}$.

We construct a network $h_2:\mathbb R^{d_x}\rightarrow\mathbb R^{d_x}$ of width $d_x+1$ by sequentially applying $g$ for each coordinate of an input $x\in\mathbb R^{d_x}$, utilizing the extra width $1$.
Then, $h_2(x)=q_K(x)$ for all $x\in[0,1]^{d_x}\setminus\mathcal D_{K,\delta,d_x}$ 
where
\begin{align*}
\mathcal D_{K,\delta,d_x}:=\{x\in\mathbb R^{d_x}:x_i\in\mathcal D_{K,\delta}~\text{for some}~i\}.
\end{align*}
Note that we use $q_K(x)$ for denoting the coordinate-wise $q_K$ for a vector $x$.

Now, we define $\mathcal D_\gamma:=([0,1]^{d_x}\setminus[\alpha,1-\alpha]^{d_x})\cup\mathcal D_{K,\delta,d_x}\subset[0,1]^{d_x}$. Then, from constructions of $h_1$ and $h_2$, we have
\begin{align*}
h_2\circ h_1(x)&=q_K(x)\qquad\qquad\qquad\qquad~\hspace{1pt}\text{for all}~x\in[0,1]^{d_x}\setminus\mathcal D_\gamma\\
    h_2\circ h_1(x)&=(1-2^{-K},\dots,1-2^{-K})~~\,\text{for all}~x\in\mathbb{R}^{d_x}\setminus[0,1]^{d_x}\\
    h_2\circ h_1(x)&\subset[0,1-2^{-K}]^{d_x}\qquad~~~\hspace{1pt}\qquad\text{for all}~x\in\mathcal D_\gamma
\end{align*}
where we use the fact that $(1,\dots,1)\notin\mathcal D_{K,\delta,d_x}$ and $q_K((1,\dots,1))=(1-2^{-K},\dots,1-2^{-K})$.

Finally, we construct a $\relu$ network $f$ of width $d_x+1$ as
\begin{align*}
    f(x):=\sum_{i=1}^{d_x}(h_2\circ h_1(x))_i\times2^{-(i-1)K}.
\end{align*}
Then, it holds that
\begin{align*}
f(x)&=\mathtt{encode}_K(x)\quad~\text{for all}~x\in[0,1]^{d_x}\setminus\mathcal D_\gamma\\
    f(x)&=1-2^{d_xK}\qquad~~\,\text{for all}~x\in\mathbb{R}^{d_x}\setminus[0,1]^{d_x}\\
    f(x)&\subset[0,1]\qquad~~\qquad\text{for all}~x\in\mathcal D_\gamma.
\end{align*}
In addition, if we choose sufficiently small $\alpha$ and $\delta$ so that $\mu(\mathcal D_\gamma)<\gamma$, then $f$ satisfies all conditions in Lemma \ref{lem:encodingrelumod}.
This completes the proof of Lemma \ref{lem:encodingrelumod}.

\subsection{Proof of Lemma \ref{lem:decodingrho}}\label{sec:pflem:decodingrho}
The proof of Lemma \ref{lem:decodingrho} is almost identical to that of Lemma \ref{lem:decodingmod}.
In particular, we approximate the $\relu$ network construction of iterative $d_y-1$ applications of $g$ (see Appendix \ref{sec:pflem:decodingmod} for the definition of $g$) by a $\rho$ network of width $d_y+1$.
To this end, we consider a $\rho$ network $h$ of width $3$ approximating $g$ on some interval $\mathcal J$ within $\alpha$ error using Lemma \ref{lem:kidger}.
Then, one can observe that iterative $d_y-1$ applications of $h$ (as in Appendix \ref{sec:pflem:decodingmod}) results in a $\rho$ network $f$ of width $d_y+1$. Here,
passing through the identity function can be approximated using a $\rho$ network of width $1$, i.e., same width to $\relu$ networks (see Lemma 4.1 by \cite{kidger19} for details).
Furthermore, since $h$ is uniformly continuous on $\mathcal J$, it holds that $\|f(c)-\mathtt{decode}_M(c)\|_\infty\le\varepsilon$ for all $c\in\mathcal C_{d_yM}$ and $f(\mathcal I)\subset[-\varepsilon,1+\varepsilon]^{d_y}$ by choosing sufficiently large $\mathcal J$ and sufficiently small $\alpha$ so that $\omega_h(\cdots\omega_h(\omega_h(\alpha)+\alpha)\cdots)+\alpha\le\varepsilon$.\footnote{We consider $\omega_h$ on $\mathcal J$.}
This completes the proof of Lemma \ref{lem:decodingrho}.

\subsection{Proof of Lemma \ref{lem:decoding2}}\label{sec:pflem:decoding2}
We first clip the input to be in $[0,1]$ using the following $\relu$ network of width $1$.
\begin{align*}
\min\big\{\max\{x,0\},1\big\}=1-\sigma(1-\sigma(x))
\end{align*}
After that, we
apply $g_\ell:[0,1]\rightarrow[0,1]^2$ defined as
\begin{align}
    (g_\ell(x))_1&:=x\notag\\
    (g_\ell(x))_2&:=\begin{cases}
    0~&\text{if}~ x\in[0,2^{-M}-\delta]\\
    \delta^{-1}2^{-M}\times(x-2^{-M}+\delta)~&\text{if}~x\in(2^{-M}-\delta,2^{-M})\\
    2^{-M}~&\text{if}~x\in[2^{-M},2\times2^{-M}-\delta]\\
    \delta^{-1}2^{-M}\times(x-2\times2^{-M}+\delta)+2^{-M}~&\text{if}~x\in(2\times2^{-M}-\delta,2\times2^{-M})\\
    &\vdots\\
    (\ell-1)\times 2^{-M}~&\text{if}~x\in[(\ell-1)\times2^{-M},1]\\
    \end{cases}\label{eq:decodeg}.
\end{align}
From the above definition of $g_\ell$, one can observe that $(g_{2^M}(x))_2=q_M(x)$ for $x\in[0,1]\setminus\mathcal D_{K,\delta}$.
Once we implement a $\relu$ network $g$ of 
width 2 such that $g(x)=g_{2^M}(x)$, then, constructing $f$ as
\begin{align*}
    f(x)&:=\big((g(z))_2,{2^M}\times\big((g(z))_1-(g(z))_2\big)\big)\\
    z&:=\min\big\{\max\{x,0\},1\big\}
\end{align*}
completes the proof.
Note that as $(g(x))_2\le x=(g(x))_1$ for all $x\in[0,1]$, $f(x)\in[0,1-2^{-M}]\times[0,1]$.
Now, we describe how to construct $g_{2^M}$ by a $\relu$ network.
One can observe that $(g_1(x))_2=0$ and
\begin{align*}
    (g_{\ell+1}(x))_2=\min\Big\{\ell\times2^{-M},\max\big\{\delta^{-1}2^{-M}\times(x-\ell\times2^{-M}+\delta)+(\ell-1)\times2^{-M},g_{\ell}(x)\big\}\Big\}
\end{align*}
for all $x$, i.e., alternating applications of $\min\{\cdot,\cdot\}$ and $\max\{\cdot,\cdot\}$.
Finally, we introduce the following definition and lemma.

\begin{definition}[\cite{hanin17}]
$f:\mathbb{R}^{d_x}\rightarrow\mathbb{R}^{d_y}$ is a max-min string of length $L$ if there exist affine transformations $h_1,\dots,h_L$ such that
$$h(x)=\tau_{L-1}(h_L(x),\tau_{L-2}(h_{L-1}(x),\cdots,\tau_2(h_3(x),\tau_1(h_2(x),h_1(x))),\cdots)$$
where each $\tau_\ell$ is either a coordinate-wise $\max\{\cdot,\cdot\}$ or $\min\{\cdot,\cdot\}$. 
\end{definition}
\begin{lemma}[{\citet[Proposition~2]{hanin17}}]\label{lem:hanin}
For any max-min string $f^*:\mathbb{R}^{d_x}\rightarrow\mathbb{R}^{d_y}$ of length $L$, for any compact $\mathcal K\subset\mathbb{R}^{d_x}$, there exists a $\relu$ network $f:\mathbb{R}^{d_x}\rightarrow\mathbb{R}^{d_x}\times\mathbb{R}^{d_y}$ of $L$ layers and width $d_x+d_y$ such that for all $x\in\mathcal K$,
\begin{align*}
    f(x) = (y_1(x), y_2(x)), \quad\text{where}\quad y_1(x) =x\quad\text{and}\quad y_2(x) =f^*(x).
\end{align*}
\end{lemma}
We note that Proposition 2 by \cite{hanin17} itself only ensures $y_2=f^*(x)$; however, its proof provides $y_1=x$ as well.

From the definition of the max-min string, one can observe that $(g_{2^M}(x))_2$ is a max-min string. Hence, by Lemma \ref{lem:hanin}, there exists a $\relu$ network $g$ of width $2$ such that $g(x)=g_{2^M}(x)=q_M(x)$ for all $x\in\mathcal D_{K,\delta}$.
This completes the proof of Lemma \ref{lem:decoding2}.

\subsection{Proof of Lemma \ref{lem:outsider}}\label{sec:pflem:outsider}
Consider the following two functions from $\mathbb R$ to $\mathbb R$:
\begin{align}
    h_1(x):=&\begin{cases}
    0~&\text{if}~x\le1-\alpha\\
    \frac1\alpha(x-1+\alpha)~&\text{if}~x\in(1-\alpha,1)\\
    1~&\text{if}~x\ge1\\
    \end{cases}\notag\\
    =&\sigma(1-\sigma(1-x)/\alpha)\notag\\
    h_2(x):=&\begin{cases}
    1~&\text{if}~x\le0\\
    \frac1\alpha(\alpha-x)~&\text{if}~x\in(0,\alpha)\\
    0~&\text{if}~x\ge\alpha\\
    \end{cases}\notag\\
    =&1-\sigma(1-\sigma(\alpha-x)/\alpha).\label{eq:encodingrelumod_h}
\end{align}
Using $h_1$ and $h_2$, we first map all $x\in\mathbb R^{d_x}\setminus[0,1]^{d_x}$ to some vector whose coordinates are greater than one via $g:\mathbb R^{d_x}\rightarrow\mathbb R^{d_x}$, defined as
\begin{align*}
g&:=r_{d_x}\circ s_{d_x}\cdots\circ r_{1}\circ s_{1}\\
r_\ell(x)&:=\sigma(x+1)-1+10\times h_1(x_\ell)\\
s_\ell(x)&:=\sigma(x+1)-1+10\times h_2(x_\ell)
\end{align*}
where we use the addition between a vector and a scalar for denoting addition of the scalar to each coordinate of the vector.
Then, one can observe that if $x\in[\alpha,1-\alpha]^{d_x}$, then $g(x)=x$ and if $x\in\mathbb R^{d_x}\setminus[0,1]^{d_x}$, then each coordinate of $g(x)$ is  greater than one.
Furthermore, each $r_\ell$ (or $s_\ell$) can be implemented by a $\relu$ network of width $d_x+1$ (width $d_x$ for computing $\sigma(x+1)-1$ and width one for computing $h_1(x_\ell)$) due to \eqref{eq:encodingrelumod_h}. Hence, $g$ can be implemented by a $\relu$ network of width $d_x+1$.

Finally, we construct a $\relu$ network $f:\mathbb R^{d_x}\rightarrow\mathbb R^{d_x}$ of width $d_x+1$ as
\begin{align*}
    f(x):=&\min\big\{\max\{g(x),0\},1\big\}\\
    \min\big\{\max\{x,0\},1\big\}=&1-\sigma(1-\sigma(x)).
\end{align*}
Then, one can observe that if $x\in[\alpha,1-\alpha]^{d_x}$, then $f(x)=x$ and if $x\in\mathbb R^{d_x}\setminus[0,1]^{d_x}$, then $f(x)=(1,\dots,1)$, and $f(\mathbb R^{d_x})\subset[0,1]^{d_x}$.
This completes the proof of Lemma \ref{lem:outsider}.

\clearpage
\section{Proofs of lower bounds}\label{sec:supplelower}
\subsection{Proof of general lower bound}\label{sec:pflb:general}
In this section, we prove that neural networks of width $d_y-1$ is not dense in both $L^p(\mathcal K,\mathbb{R}^{d_y})$ and $C(\mathcal K,\mathbb{R}^{d_y})$, regardless of the activation functions.
\begin{lemma}\label{lem:dylb}
For any set of activation functions, networks of width $d_y-1$ are not dense in both $L^p(\mathcal K,\mathbb{R}^{d_y})$ and $C(\mathcal K,\mathbb{R}^{d_y})$.
\end{lemma}
\begin{proof}
In this proof, we show that networks of width $d_y-1$ are not dense in $L^p([0,1]^{d_x},\mathbb{R}^{d_y})$, which can be easily generalized to the cases of $L^p(\mathcal K,\mathbb{R}^{d_y})$ and hence, to $C(\mathcal K,\mathbb{R}^{d_y})$.
In particular, we prove that there exist $f^*\in L^p([0,1]^{d_x},\mathbb{R}^{d_y})$ and $\varepsilon>0$ such that for any network $f$ of width $d_y-1$, it holds that
\begin{align*}
    \|f^*-f\|_p>\varepsilon. 
\end{align*}

Let $\Delta$ be a $d_y$-dimensional regular simplex with sidelength $\sqrt{2}$, that is isometrically embedded into $\mathbb{R}^{d_y}$. The volume of this simplex is given as $\textup{Vol}_{d_y}(\Delta) = \frac{\sqrt{d_y+1}}{d_y!}$.\footnote{$\textup{Vol}_{d}(\mathcal S)$ denotes the volume of $\mathcal S$ in the $d$-dimensional Euclidean space.} We denote the vertices of this simplex by $\{v_0,\ldots,v_{d_y}\}$. Then, we can construct the counterexample as follows.
\begin{align*}
    f^*(x) = \begin{cases} v_i &\hspace{-7pt}\text{if}~x_1\in\big[\frac{2i}{2d_y+1},\frac{2i+1}{2d_y+1}\big] \text{ for some } i\\
    (2d_y+1)(v_{i+1}-v_i)x_1+(2i+2)v_i-(2i+1)v_{i+1}&\hspace{-7pt}\text{if}~ x_1\in \big[\frac{2i+1}{2d_y+1},\frac{2i+2}{2d_y+1}\big] \text{ for some } i
     \end{cases}.
\end{align*}
In other words, $f^*(x)$ travels the vertices of $\Delta$ sequentially as we move $x_1$ from $0$ to $1$, staying at each vertex over an interval of length $\frac{1}{2d_y+1}$ and traveling between vertices at a constant speed otherwise, i.e., $f^*$ is continuous and in $L^p([0,1]^{d_x},\mathbb R^{d_y})$.

Recalling \eqref{eq:neuralnet}, one can notice that any neural network $f$ of width less than $d_y$ and $L\ge2$ layers can be decomposed as $t_L \circ g$, where $t_L: \reals^{k} \to \reals^{d_y}$ is the last affine transformation and $g$ denotes all the preceding layers, i.e., $g = \sigma_{L-1} \circ t_{L-1} \circ \cdots \circ \sigma_1 \circ t_1$. 
Here, we consider $k={d_y-1}$ as it suffices to cover cases $k<{d_y-1}$.
Now, we proceed as
\begin{align*}
    \|f^*-f\|_p&=\left(\int_{[0,1]^{d_x}} \|f^*(x) - f(x)\|_p^p ~dx\right)^{\frac1p}\\
    &\geq \left(\int_{[0,1]^{d_x}} \|f^*(x) - t_L\circ g(x)\|_p^p ~dx\right)^{\frac1p}\\
    &\geq \left(\int_{[0,1]^{d_x}} \inf_{u^*(x) \in \mathbb{R}^{d_y-1}}\|f^*(x) - t_L(u^*(x))\|_p^p ~dx\right)^{\frac1p}\\
    &\geq \left(\frac{1}{2d_y+1}\right)^{\frac1p}\inf_{t: \text{ affine map}} \max_{i \in [d_y+1]} \inf_{u^*_i\in\mathbb{R}^{d_y-1}}\|v_i - t(u^*_i)\|_p\\
    &\geq \left(\frac{1}{2d_y+1}\right)^{\frac1p}\inf_{\mathcal{H} \in \mathfrak{H}} \max_{i \in [d_y+1]} \inf_{a_i \in \mathcal{H}}\|v_i - a_i\|_p,
\end{align*}
where $\mathfrak{H}$ denotes the set of all $(d_y-1)$-dimensional hyperplanes in $\reals^{d_y}$ and $[d_y+1]:=\{0,1,\dots,d_y\}$. As the vertices of $\Delta$ are $d_y + 1$ distinct points in a general position, $\inf_{\mathcal{H} \in \mathfrak{H}} \max_{i \in [d_y+1]} \inf_{a_i \in \mathcal{H}} \|v_i - a_i\|_p > 0$. To make this argument more concrete, we take a volumetric approach; for any $k$-dimensional hyperplane $\mathcal{H} \in \mathbb{R}^{d_y}$, we have
\begin{align*}
    \textup{Vol}_{d_y}(\Delta) \le 2 \cdot \textup{Vol}_{d_y-1}(\pi_\mathcal{H}(\Delta)) \cdot \max_{i \in [d_y+1]}\inf_{a_i\in\mathcal{H}}\|v_i - a_i\|_2,
\end{align*}
where $\pi_\mathcal{H}$ denotes the projection onto $\mathcal{H}$. As projection is contraction and the distance between any two points are at most $\sqrt{2}$, it holds that for any $\mathcal{H}$,
\begin{align*}
    \max_{i \in [d_y+1]}\inf_{a_i\in\mathcal{H}}\|v_i - a_i\|_2 &\ge \frac{\textup{Vol}_{d_y}(\Delta)}{2 \cdot \textup{Vol}_{d_y-1}\left(\{x\in\mathbb{R}^{d_y-1}:\|x\|_2\le1\}\right)}\\
    &= \frac{\sqrt{d_y + 1}}{2 \cdot d_y !} \cdot \Gamma\left(\frac{d_y + 1}{2}\right)\cdot \left(\frac{2}{\pi}\right)^{\frac{d_y-1}{2}}
\end{align*}
where $\Gamma$ denotes the gamma function, and we use the fact that $\textup{Vol}_{d_y-1}\left(\{x\in\mathbb{R}^{d_y-1}:\|x\|_2\le1\}\right)\ge\textup{Vol}_{d_y-1}(\pi_\mathcal{H}(\Delta))$ as $\Delta$ can be contained in a $d_y$-dimensional unit ball, and hence $\pi_\mathcal{H}(\Delta)$ can be contained in a $(d_y-1)$-dimensional unit ball.
Thus we have $\|f^*-f\|_p>\varepsilon$ with
\begin{align*}
    \varepsilon = \frac{d_y^{\frac{1}{p}-\frac{1}{2}}\sqrt{d_y + 1}}{2\cdot(2d_y + 1)^{\frac1p} \cdot d_y !} \cdot \Gamma\left(\frac{d_y + 1}{2}\right)\cdot \left(\frac{2}{\pi}\right)^{\frac{d_y-1}{2}},
\end{align*}
for $p \ge 2$ and
\begin{align*}
    \varepsilon = \frac{\sqrt{d_y + 1}}{2\cdot(2d_y + 1)^{\frac1p} \cdot d_y !} \cdot \Gamma\left(\frac{d_y + 1}{2}\right)\cdot \left(\frac{2}{\pi}\right)^{\frac{d_y-1}{2}},
\end{align*}
for $p < 2$.
This completes the proof of Lemma \ref{lem:dylb}.
\end{proof}

\subsection{Proof of tight lower bound in Theorem \ref{thm:reluthresunif}}\label{sec:pfthm:reluthresuniflb}
In this section, we prove the tight lower bound in Theorem \ref{thm:reluthresunif}.
Since we already have the width-$d_y$ lower bound by Lemma \ref{lem:dylb} and it is already proven that $\relu$ networks of width $d_x$ is not dense in $C(\mathcal K,\mathbb R^{d_y})$ \citep{hanin17}, we prove the tight lower bound in Theorem \ref{thm:reluthresunif} by showing the following statement: There exist $f^*\in C([0,1]^{d_x},\mathbb R)$ and $\varepsilon>0$ such that for any $\relu$+$\step$ network $f$ of width $d_x$ containing at least one $\step$, it holds that
\begin{align*}
    \|f^*-f\|_\infty>\varepsilon.
\end{align*}
Without loss of generality, we assume that $f$ has $d_x$ hidden neurons at each layer except for the output layer and all affine transformations in $f$ are invertible (see Section \ref{sec:topologicalproerty}).

Our main idea is to utilize properties of \textit{level sets} of width-$d_x$ $\relu$+$\step$ networks \citep{hanin17} defined as follows:
Given a network $f$ of width $d_x$, we call a connected component of $f^{-1}(y)$ for some $y$ as a level set.
Level sets of $\relu$+$\step$ networks have a property described by the following lemma.
We note that the statement and the proof of Lemma \ref{lem:level} is motivated by Lemma 6 of \citep{hanin17}.
\begin{lemma}\label{lem:level}
Let $f$ be a $\step$+$\relu$ network of width $d_x$ containing at least one $\step$. Then, for any level set $\mathcal S$ of $f$, $\mathcal S$ is unbounded unless it is empty.
\end{lemma}
\begin{proof}[Proof of Lemma \ref{lem:level}]
Let $\ell^*$ be the smallest number such that $\step$ appears at the $\ell^*$-th layer. 
In this proof, we show that all level sets of the first $\ell^*$ layers of $f$ are either unbounded or empty.
Then the claim of Lemma \ref{lem:level} directly follows.
We prove this using the mathematical induction on $\ell^*$.
Recalling \eqref{eq:neuralnet}, we denote by $f_\ell$ the mapping of the first $\ell$ layers of $f$:
\begin{align*}
    f_\ell:=\sigma_\ell \circ t_\ell\circ\cdots\circ\sigma_1\circ t_1.
\end{align*}
First, consider the base case: $\ell^*=1$. Assume without loss of generality that the activation function of the first hidden node in $\sigma_1$ is $\step$.
Then for any $x$, the $\step$ activation maps the first component of $t_1(x)$ to $1$ if $(t_1(x))_1 \geq 0$, and to $0$ otherwise.
This means that there exists a ray $\mc R$
starting from $x$ such that $f_1(\mathcal R)=\{f_1(x)\}$. Hence, any level set of $f_1$ is either unbounded or empty.

Now, consider the case that $\ell^*>1$.
Then, until the $(\ell^*-1)$-th layer, the network only utilizes $\relu$. Here, the level sets of $f_{\ell^*-1}$ can be characterized using the following lemma.
\begin{lemma}[{\citet[Lemma~6]{hanin17}}]\label{lem:haninlevel}
Given a $\relu$ network $g$ of width $d_x$, let $\mathcal S\subset\mathbb R^{d_x}$ be a set such that $x\in\mathcal S$ if and only if inputs to all $\relu$ in $g$ are strictly positive, when computing $g(x)$.
Then, $\mathcal S$ is open and convex, $g$ is affine on $\mathcal S$, and any bounded level set of $g$ is contained in $\mathcal S$.
\end{lemma}
Consider $\mathcal S$ of $f_{\ell^*-1}$ as in Lemma \ref{lem:haninlevel} and consider a level set $\mathcal T$ of $f_{\ell^*}$ containing some $x$, i.e., $\mathcal T\ne\emptyset$.
If $x\notin\mathcal S$, then $\mathcal T$ is unbounded by Lemma \ref{lem:haninlevel}.
If $x\in\mathcal S$, we argue as the base case.
The preimage of $f_{\ell^*}(x)$ of the $\ell^*$-th layer (i.e., $\sigma\circ t_{\ell^*}$) contains a ray.
If this ray is contained in $f_{\ell^*-1}(\mathcal S)$, then $\mathcal T$ is unbounded as $f_{\ell^*-1}$ is invertible and affine on $\mathcal S$.
Otherwise, $\mathcal T\setminus\mathcal S\ne\emptyset$ and it must be unbounded as any level set of $f_{\ell^*-1}$ not contained in $\mathcal S$ is unbounded by Lemma \ref{lem:haninlevel}.
This completes the proof of Lemma \ref{lem:level}.
\end{proof}
Now, we continue the proof of the tight lower bound in Theorem \ref{thm:reluthresunif} based on Lemma \ref{lem:level}.
We note that our argument is also from the proof of the lower bound in Theorem 1 of \citep{hanin17}.

Consider $f^*:[0,1]^{d_x}\rightarrow\mathbb R$ defined as
\begin{align*}
    f^*(x_1,\dots,x_{d_x}):=\sum_{i=1}^{d_x}\left(x_i-\frac12\right)^2.
\end{align*}
Then, for $a=\frac14$ and $b=0$, one can observe that $(f^*)^{-1}(a)$ is a sphere of radius $\frac{1}{2}$ centered at $(f^*)^{-1}(b)=\{(\frac12,\dots,\frac12)\}$.
Namely, any path from $(f^*)^{-1}(b)$ to infinity must intersect with $(f^*)^{-1}(a)$.
Now, suppose that a $\relu$+$\step$ network $f$ of width $d_x$ satisfies that $\|f^*-f\|_\infty\le\frac1{16}$.
Then, the level set of $f$ containing $(\frac12,\dots,\frac12)$ must be unbounded by Lemma \ref{lem:level}, and hence, must intersect with $(f^*)^{-1}(a)$.
However, as $f^*\circ(f^*)^{-1}(a)=\frac14$ and $f^*\circ(f^*)^{-1}(b)=0$, this contradicts with $\|f^*-f\|_\infty\le\frac1{16}$.
This completes the proof of the tight lower bound $\max\{d_x+1,d_y\}$ in Theorem \ref{thm:reluthresunif}.

\subsection{Proof of tight lower bound in Theorem \ref{thm:relulp}}\label{sec:pfthm:relulplb}
In this section, we prove the tight lower bound in Theorem \ref{thm:relulp}.
Since we already have the $d_y$ lower bound by Lemma \ref{lem:dylb}, we prove the tight lower bound in Theorem \ref{thm:relulp} by showing the following statement: There exist $f^*\in L^p(\mathbb R^{d_x},\mathbb R)$ and $\varepsilon>0$ such that for any continuous function $f$ represented by a $\relu$ network of width $d_x$, it holds that
\begin{align*}
    \|f^*-f\|_\infty>\varepsilon.
\end{align*}
Note that this statement can be easily generalized to an arbitrary codomain.
To derive the statement, we prove a stronger statement: For any $\relu$ network $f$ of width $d_x$, either
\begin{align}
f\notin L^p(\mathbb R^{d_x},\mathbb R)\quad\text{or}\quad f=0\label{eq:lplevel}
\end{align}
where $f=0$ denotes that $f$ is a constant function mapping any input to zero.
Then it leads us to the desired result directly.
Without loss of generality, we assume that $f$ has $d_x$ hidden neurons at each layer except for the output layer and all affine transformations in $f$ are invertible (see Section \ref{sec:topologicalproerty}).

As in the proof of the tight upper bound in Theorem \ref{thm:reluthresunif}, we utilize properties of level sets of $f$ given by Lemma \ref{lem:haninlevel}.
Let $\mathcal S$ be a set defined in Lemma \ref{lem:haninlevel} of $f$.
By the definition of $\mathcal S$, one can observe that $\mathbb R^{d_x}\setminus\mathcal S\ne\emptyset$. Then, a level set $\mathcal T$ containing some $x\in\mathbb R^{d_x}\setminus\mathcal S$ must be unbounded by Lemma \ref{lem:haninlevel}.
Here, if $y:=f(x)>0$, then for $\delta:=\omega_f^{-1}(\frac{y}2)$, we have $f(x^\prime)\ge\frac{y}2$ for all \begin{align*}
x^\prime\in\mathcal T^\prime:=\{x^\prime\in\mathbb R^{d_x}:\exists x\in\mathcal T~\text{such that}~\|x^\prime-x\|_\infty\le\delta\}.
\end{align*}
Since $\mc T^\prime$ contains $\mc T$ which is an unbounded set, one can easily observe that $\mu(\mathcal T^\prime)=\infty$ and hence, $\int_{\mathcal T^\prime}|f(x)|^pd(x)=\infty$, i.e., $f\notin L^p(\mathbb R^{d_x},\mathbb R)$.\footnote{$\mu$ denotes the Lebesgue measure.}
One can derive the same result for $f(x)<0$.

Suppose that $f(x)=0$ for all $x\in\mathbb R^{d_x}\setminus\mathcal S$. Then, 
$f(x)=0$ for all $x\in\partial\mathcal S$ as $\mathcal S$ is open (see Lemma \ref{lem:haninlevel}).
Furthermore, we claim that $f(\mathcal S)=\{0\}$ or $f\notin L^p(\mathbb R^{d_x},\mathbb R)$.
For any $s\in\mathcal S$, consider any two rays of opposite directions starting from $s$.
If one ray is contained in $\mathcal S$ and $f\in L^p(\mathbb R^{d_x},\mathbb R)$, then its image for $f$ must be $\{0\}$. If the image of $f$ is not $\{0\}$, using the similar argument for the case $f(x)>0$ leads us to $f\notin L^p(\mathbb R^{d_x},\mathbb R)$.
Then, one can conclude that $f(s)=0$.
If both rays are not contained in $\mathcal S$, they must both intersect with $\partial\mathcal S$.
Then, since $f$ is affine on $\mathcal S$, $f(s)$ must be zero as $f(\partial\mathcal S)=\{0\}$.
Hence, we prove the claim.

This completes the proof of the tight upper bound in Theorem \ref{thm:relulp}.

\subsection{Proof of Lemma \ref{lem:relu}}\label{sec:pflem:relu}
Let $\phi(x)=Ax+b$ for some invertible matrix $A=[a_1,a_2]^\top\in\mathbb{R}^{2\times2}$ and for some vectors $a_1,a_2,b\in\mathbb R^2$.
Then, it is easy to see that if \begin{align*}
\langle a_1,x\rangle+b_1\ge0\quad\text{and}\quad\langle a_2,x\rangle+b_2\ge0,
\end{align*}
i.e., $x\in\mathcal S$, then $\phi^{-1}\circ\sigma\circ\phi(x)=x$. Hence, the first statement of Lemma \ref{lem:relu} holds.

Now, consider the second statement of Lemma \ref{lem:relu}.
Suppose that $\langle a_1,x\rangle+b_1\ge0$ but $\langle a_2,x\rangle+b_2<0$.
Then, one can easily observe that $\phi^{-1}\circ\sigma\circ\phi$ maps a ray 
\begin{align*}
\{x^\prime\in\mathbb R^2:\langle a_1,x^\prime\rangle=\langle a_1,x\rangle,\langle a_2,x^\prime\rangle+b_2<0\}
\end{align*}
containing $x$ to a single point $\phi^{-1}\big((\langle a_1,x\rangle+b_1,0)\big)$, which is on $\partial\mathcal S$.
In addition, similar arguments hold for cases that $\langle a_1,x\rangle+b_1<0$, $\langle a_2,x\rangle+b_2\ge0$ and $\langle a_1,x\rangle+b_1<0$, $\langle a_2,x\rangle+b_2<0$.
This completes the proof of Lemma \ref{lem:relu}.

\subsection{Proof of Lemma \ref{lem:topology}}\label{sec:pflem:topology}

We first prove the first statement of Lemma
\ref{lem:topology} using the proof by contradiction.
Suppose that $x^\prime=x$ and $x^\prime\notin\mathcal T^\prime$ but $x^\prime$ is not in a bounded path-connected component of $\mathbb R^2\setminus\mathcal T^\prime$.
Here, note that $x=x^\prime\in\mathcal S$ for $\mathcal S$ defined in Lemma \ref{lem:relu}.
Then, there exists a path $\mathcal P$ from $x^\prime$ to infinity such that $\mathcal P\cap\mathcal T^\prime=\emptyset$.
If $\mathcal P\subset\text{int}(\mathcal S)$\footnote{$\text{int}(\mathcal S)$ denotes the interior of $\mathcal S$.}, then the preimages of $\mathcal P$ and $\mathcal T^\prime\cap\text{int}(\mathcal S)$ under $\phi^{-1}\circ\sigma\circ\phi$ stay identical to their corresponding images, i.e., $\mathcal P$ and $\mathcal T^\prime\cap\text{int}(\mathcal S)$ (by Lemma \ref{lem:relu}). This contradicts the assumption that $x$ is in a bounded path-connected component of $\mathbb R^2\setminus\mathcal T$.
Hence, it must hold that $\mathcal P\not\subset\text{int}(\mathcal S)$.

Let $x^*\notin\mathcal T^\prime$ be the first point in $\mathcal P\cap\partial\mathcal S$ in the trajectory of $\mathcal P$ starting from $x^\prime$.
Then, the preimage of $x^*$ contains a ray $\mc R$ starting from $x^*$ (see the proof of Lemma \ref{lem:relu} for the details) which must not intersect with $\mathcal T$; had the ray $\mc R$ intersected with $\mc T$, then $\mc R \cap \mc T$ must have mapped to $x^*$, which contradicts $x^* \notin \mc T^\prime$ and the definition of $\mc P$.
Furthermore, from the definition of $x^*$, the subpath $\mc P^\dagger$ of $\mathcal P$ from $x^\prime$ to $x^*$ excluding $x^*$ satisfies $\mc P^\dagger \subset \text{int}(\mc S)$. Hence, the preimages of $\mc P^\dagger$ and $\mathcal T^\prime\cap\text{int}(\mathcal S)$ under $\phi^{-1}\circ\sigma\circ\phi$ stay identical by Lemma \ref{lem:relu}.
This implies that there exist a path $\mc P^\dagger$ from $x$ to $x^*$, and then a path $\mc R$ from $x^*$ to infinity, not intersecting with $\mathcal T$. This contradicts the assumption of Lemma \ref{lem:topology}.
This completes the proof of the first statement of Lemma \ref{lem:topology}.

Now, consider the second statement of Lemma \ref{lem:topology}.
By Lemma \ref{lem:relu}, $x\ne x^\prime$ implies that $x\notin\mathcal S$ and $x^\prime\in\partial\mathcal S$.
Here, as the preimage of $x^\prime$ contains a ray from $x^\prime$ containing $x$, this ray must intersect with $\mathcal T$ from the assumption of Lemma \ref{lem:topology}.
Hence, $x^\prime\in\mathcal T^\prime$ and this completes the proof of the second statement of Lemma \ref{lem:topology}.

By combining the proofs of the first and the second statements of Lemma \ref{lem:topology}, we complete the proof of Lemma \ref{lem:topology}.

\subsection{Proof of Lemma \ref{lem:topology2}}\label{sec:pflem:topology2}
Before starting our proof, we first introduce the following definitions and lemma. The proof of Lemma \ref{lem:topologypolygon} is presented in Appendix \ref{sec:pflem:topologypolygon}.
\begin{definition}
Definitions related to curves, loops, and polygons are listed as follows: For $\mathcal U\subset\mathbb R^2$ and $\mathcal F(\mathcal U):=\{f\in C([0,1],\mathbb R^2):f([0,1])=\mathcal U\}$,
\begin{list}{{\tiny$\bullet$}}{\leftmargin=1.8em}
  \setlength{\itemsep}{1pt}
  \vspace*{-4pt}
    \item $\mathcal U$ is a ``curve'' if there exists $f\in \mathcal F(\mathcal U)$.
    \item $\mathcal U$ is a  ``simple curve'' if there exists injective $f\in \mathcal F(\mathcal U)$.
    \item $\mathcal U$ is a ``loop'' if there exists $f\in \mathcal F(\mathcal U)$ such that $f(1)=f(0)$.
    \item $\mathcal U$ is a ``simple loop'' if there exists $f\in \mathcal F(\mathcal U)$ such that $f(1)=f(0)$ and $f$ is injective on $[0,1)$.
    
    \item $\mathcal U$ is a ``polygon'' if there exists piece-wise linear $f\in \mathcal F(\mathcal U)$ such that $f(1)=f(0)$.
    
    \item $\mathcal U$ is a ``simple polygon'' if there exists piece-wise linear $f\in \mathcal F(\mathcal U)$ such that $f(1)=f(0)$ and $f$ is injective on $[0,1)$.
\end{list}
\end{definition}
\begin{lemma}\label{lem:topologypolygon}
Suppose that $g_{\ell^*}([0,p_1])\cap g_{\ell^*}([p_2,1])=\emptyset$ and $g_{\ell^*}([0,p_1])\setminus\mathcal B$ is contained in a bounded path-connected component of $\mathbb R^2\setminus\mathcal U$ for some $\mathcal U\subset g_{\ell^*}([p_2,1])\cup\mathcal B$.
Then, $g_{\ell^*}([0,p_1])\setminus\mathcal B$ is contained in a bounded path-connected component of $\mathbb R^2\setminus(g_{\ell^*}([p_2,1])\cup\mathcal B)$.
\end{lemma}
In this proof, we prove that  if $g_{\ell^*}([0,p_1])\cap g_{\ell^*}([p_2,1])=\emptyset$, then $g_{\ell^*}([0,p_1])\setminus\mathcal B$ is contained in a bounded path-connected component of $\mathbb R^2\setminus\mathcal U$ for some simple polygon $\mathcal U\subset g_{\ell^*}([p_2,1])\cup\mathcal B$.
Then, the statement of Lemma \ref{lem:topology2} directly follows by Lemma \ref{lem:topologypolygon}.

To begin with, consider a loop  $g_{\ell^*}([p_2,1])\cup\mathcal L\big(g_{\ell^*}(p_2),g_{\ell^*}(1)\big)$ where $\mathcal L(x,x^\prime)$ denotes the line segment from $x$ to $x^\prime$, i.e.,
\begin{align*}
    \mathcal L(x,x^\prime):=\{\lambda\cdot x+(1-\lambda)\cdot x^\prime:\lambda\in[0,1]\}.
\end{align*}
Then, the loop consists of a finite number of line segments, as an image of an interval of a $\relu$ network is piece-wise linear as well as $\mathcal L\big(g_{\ell^*}(p_2),g_{\ell^*}(1)\big)$, i.e., the loop is a polygon.

Since $g_{\ell^*}([p_2,1])\cup\mathcal L\big(g_{\ell^*}(p_2),g_{\ell^*}(1)\big)$ consists of line segments, under the assumption $\|f^*-f\|_\infty\le\frac1{100}$, one can easily construct a simple loop $\mathcal U$ in $g_{\ell^*}([p_2,1])\cup\mathcal L\big(g_{\ell^*}(p_2),g_{\ell^*}(1)\big)$ so that $\mathcal U$ contains simple curves from the midpoint of $\mathcal L\big(g_{\ell^*}(p_2),g_{\ell^*}(1)\big)$ to a point near the point $(-1,1)$, then to a point near the point $(1,1)$, and finally to the midpoint of $\mathcal L\big(g_{\ell^*}(p_2),g_{\ell^*}(1)\big)$.
We note that $\mathcal U$ also consists of line segments, i.e., $\mathcal U$ is a simple polygon.
Figure \ref{fig:topology1} illustrates $\mathcal U$ where line segments from $g_{\ell^*}([p_2,1])$ is drawn in blue and line segments from $L\big(g_{\ell^*}(p_2),g_{\ell^*}(1)\big)$ indicated by dotted black line.

Now, choose $q\in(0,p_1)$ such that $f^*(q)=(0,\frac12)$. Since $\|f^*-f\|_\infty\le\frac1{100}$ and by the definition of $\ell^*$, 
\begin{align*}
f(q)=g_{\ell^*}(q)\in\{x \in \mathbb R^2:\|x-(0,\tfrac12)\|_\infty\le\tfrac1{100}\}
\end{align*}
which is illustrated by the red dot in Figure \ref{fig:topology}. 
Then, we claim the following statement:
\begin{align}
\text{$g_{\ell^*}(q)$ is contained in a bounded path-connected component of $\mathbb R^2\setminus\mathcal U$.}\label{eq:claim}
\end{align}
From the definition of $q$ and the path-connectedness of $g_{\ell^*}([0,q])$, one can observe that proving the claim \eqref{eq:claim} leads us to that $g_{\ell^*}([0,q])$ is contained in a bounded path-connected component of $\mathbb R^2\setminus\mathcal U$ unless $\mathcal U\cap g_{\ell^*}([0,q])\ne\emptyset$.
Since $g_{\ell^*}([0,p_1])\setminus\mathcal B\subset g_{\ell^*}([0,q])$ by the definitions of $q,\ell^*$ and the assumption that $\|f^*-f\|_\infty\le\frac1{100}$, 
this implies that if $g_{\ell^*}([0,p_1])\cap g_{\ell^*}([p_2,1])=\emptyset$, then $g_{\ell^*}([0,p_1])\setminus\mathcal B$ is contained in a bounded path-connected component of $\mathbb R^2\setminus\mathcal U$.
Hence, \eqref{eq:claim} implies the statement of Lemma \ref{lem:topology2}.

To prove the claim \eqref{eq:claim}, we first introduce the following lemma.
\begin{figure*}[t]
\centering
\subfigure[]{
\centering
\includegraphics[width=0.25\textwidth]{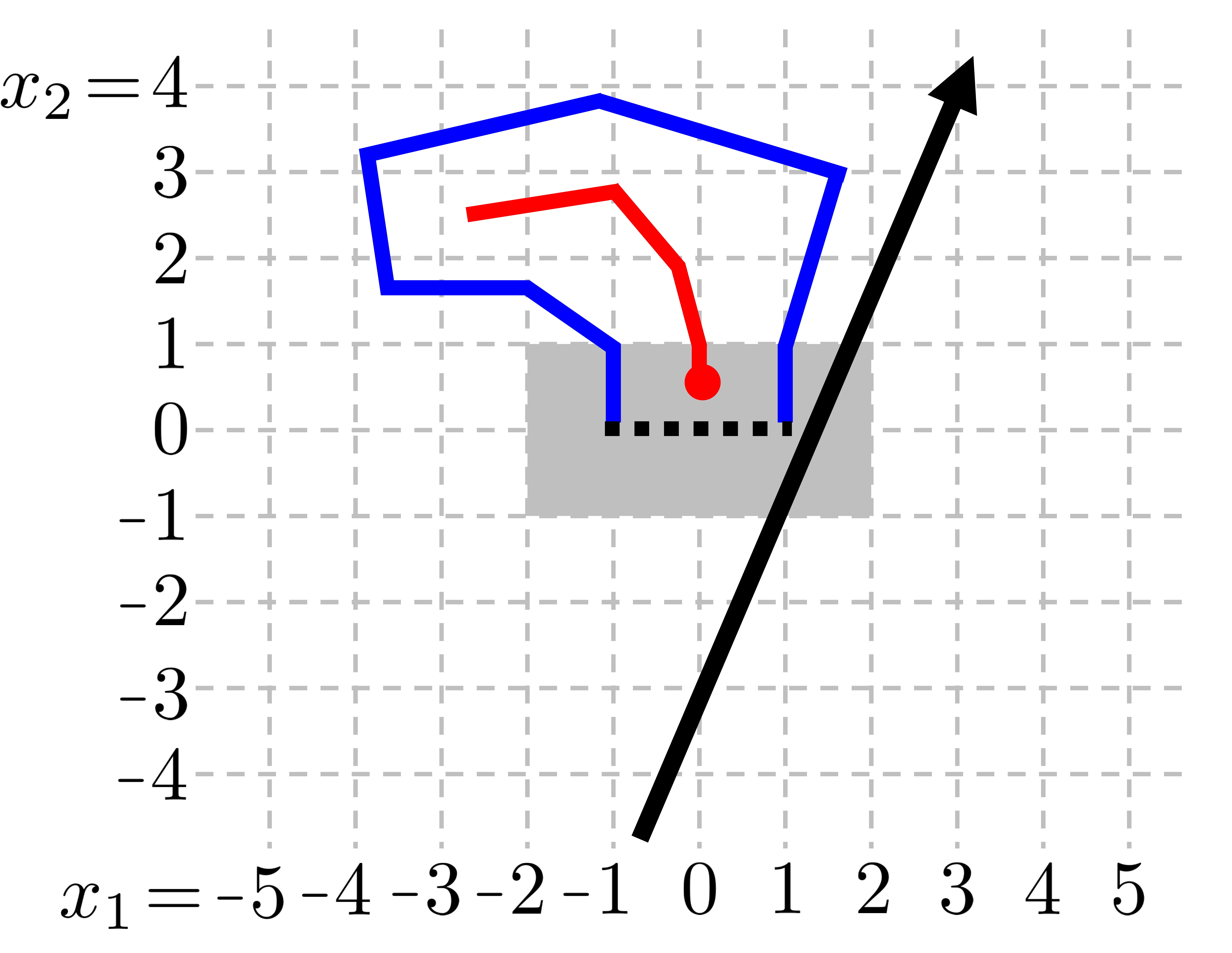}
\label{fig:topology1}
}
\subfigure[]{
\centering
\includegraphics[width=0.25\textwidth]{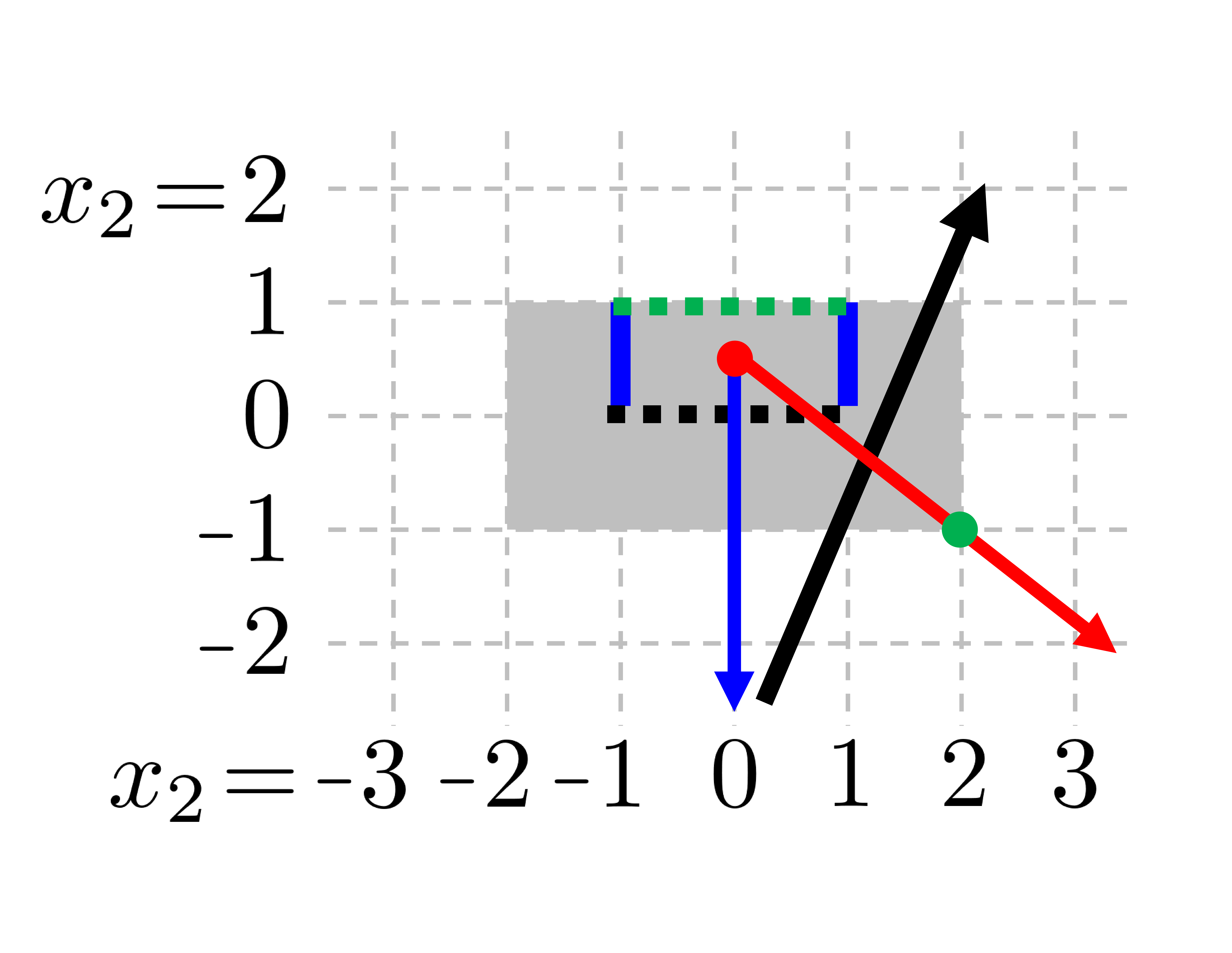}
\label{fig:topology2}
}
\subfigure[]{
\centering
\includegraphics[width=0.25\textwidth]{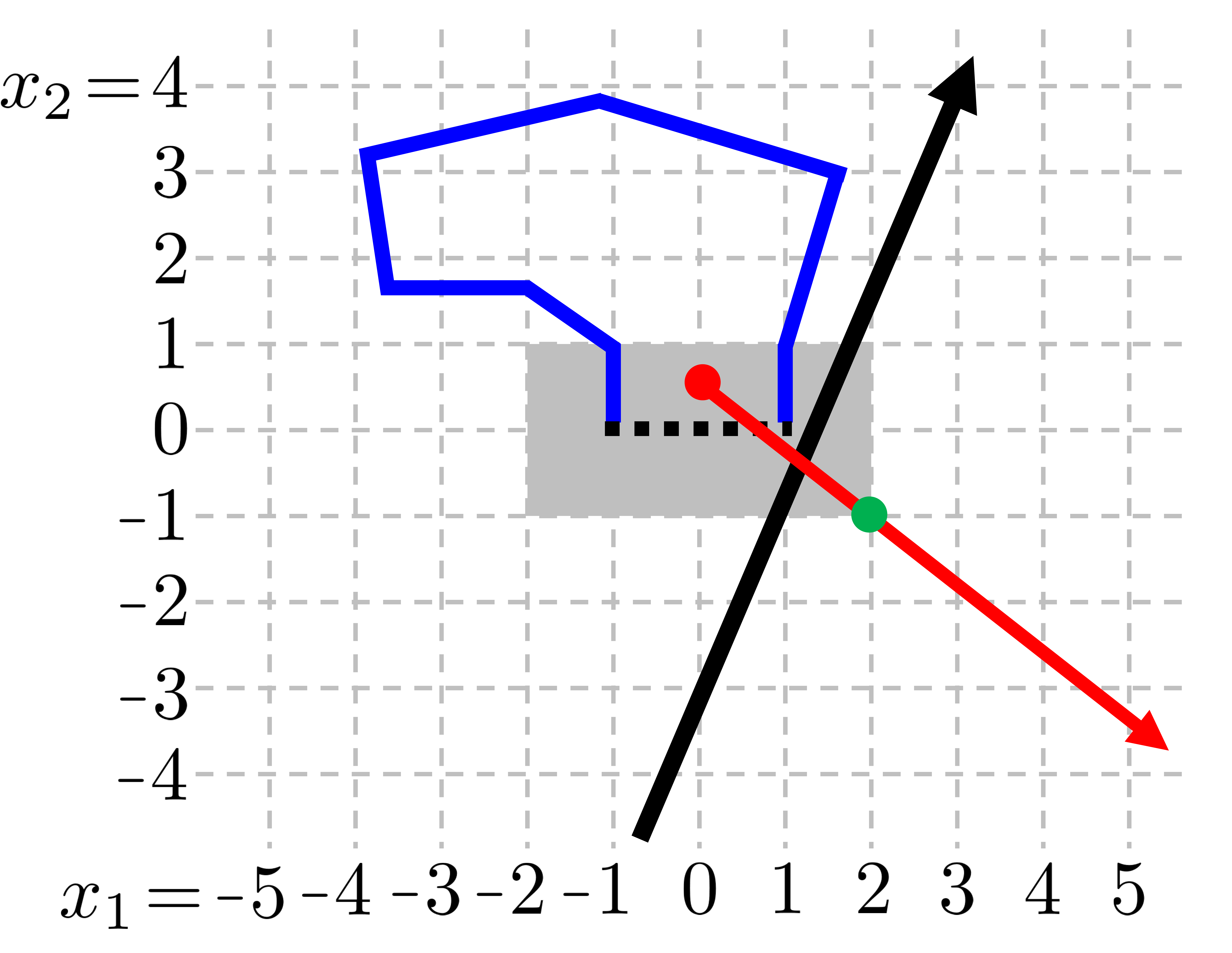}
\label{fig:topology3}
}

\caption{(a) Illustration of $\mathcal U$, $g_{\ell^*}(q)$. (b) Illustration of $\mathcal V$, $g_{\ell^*}(q)$, $v$. (c) Illustration of $\mathcal U$, $g_{\ell^*}(q)$, $v$.}
\label{fig:topology}
\end{figure*}
\begin{lemma}[Jordan curve theorem \citep{tverberg80}]\label{lem:jordancurve}
For any simple loop $\mathcal O\subset\mathbb R^2$, $\mathbb R^2\setminus\mathcal O$ consists of exactly two path-connected components where one is bounded and another is unbounded. 
\end{lemma}
Lemma \ref{lem:jordancurve} ensures the existence of a bounded path-connected component of $\mathbb R^2\setminus\mathcal U$.

Furthermore, to prove the claim \eqref{eq:claim}, we introduce the parity function $\pi_{\mathcal U}:\mathbb{R}^2\setminus\mathcal U\rightarrow \{0,1\}$: For $x\in\mathbb R^2\setminus\mathcal U$ and a ray starting from $x$, $\pi_{\mathcal U}(x)$ counts the number of times that the ray ``properly'' intersects with $\mathcal U$ (reduced modulo 2) where the proper intersection is an intersection where $\mathcal U$ enters and leaves on different sides of the ray.
Here, it is well-known that $\pi_{\mathcal U}(x)$ does not depend on the choice of the ray, i.e., $\pi_{\mathcal U}$ is well-defined.
We refer the proof of Lemma 2.3 by \cite{thomassen92} and the proof of Lemma 1 by \cite{tverberg80} for more details.
Here, $\pi_{\mathcal U}$ characterizes the ``position'' of $x$ as $\pi_{\mathcal U}(x)=0$ if and only if $x$ is in the unbounded path-connected component of $\mathbb R^2\setminus\mathcal U$, which is known as the even-odd rule \citep{shimrat62,hacker62}.
Hence, proving that $\pi_{\mathcal U}(g_{\ell^*}(q))=1$ would complete the proof of the claim \eqref{eq:claim}.

Recall that there exists the line (e.g., the black arrow in Figure \ref{fig:topology}) that intersects with $\mathcal B$ and the image of $g_{\ell^*}$ can be at only ``one side'' of the line (see Section \ref{sec:counter} for details).
Since $\mathcal B$ is open, there exists a ``vertex'' $v\in\partial\mathcal B$ (e.g., the green dot in Figures \ref{fig:topology2} and \ref{fig:topology3}) such that $v$ is in the ``other side'' of the line.\footnote{A vertex denotes one of the points $(2,-1),(2,1),(-2,-1),(-2,1)$.}
We prove $\pi_{\mathcal U}(g_{\ell^*}(q))=1$ by counting the number of proper intersections between the ray $\mathcal R$ from $g_{\ell^*}(q)$ passing through $v$ (the red arrow in Figures \ref{fig:topology2} and \ref{fig:topology3} illustrates $\mathcal R$).

To simplify showing $\pi_{\mathcal U}(g_{\ell^*}(q))=1$, we consider
two points $z_1,z_2\in\mathcal U\cap\partial B$ near the points $(-1,1), (1,1)$, respectively, such that the simple curve $\mathcal P$ in $\mathcal U$ from $z_1$ to $z_2$ is contained in $\mathcal B$ except for $z_1,z_2$.
Then, one can observe that $\mathcal P$ and $\mathcal L(z_1,z_2)$ forms a simple loop which we call $\mathcal V$. 
Figure \ref{fig:topology2} illustrates $\mathcal V$ where the black dotted line indicates the line segment from $\mathcal L\big(g_{\ell^*}(p_2),g_{\ell^*}(1)\big)$, the blue line indicates the line segments from $g_{\ell^*}([p_2,1])$, and the green dotted line indicates $\mathcal L(z_1,z_2)$; from the definition of $\mc P$, the blue and green lines together correspond to $\mc P$.

Then, $\pi_{\mathcal V}(g_{\ell^*}(q))=1$ as a ray from $g_{\ell^*}(q)$ of the downward direction (the blue arrow in Figure \ref{fig:topology2}) only properly intersects once with $\mathcal V$ at some point in $\mathcal L\big(g_{\ell^*}(p_2),g_{\ell^*}(1)\big)$, under the assumption that $\|f^*-f\|_\infty\le\frac1{100}$.
From the property of $\pi_{\mathcal V}$, this implies that the ray $\mathcal R$ starting from $g_{\ell^*}(q)$ and passing through $v$ (e.g., the red arrow in Figures \ref{fig:topology2} and \ref{fig:topology3}) must properly intersect with $\mathcal V$ odd times.
Furthermore, from the construction of $\mathcal U$ and $\mathcal V$, definition of $\ell^*$, and under the assumption that $\|f^*-f\|_\infty\le\frac1{100}$, one can observe that the simple curve in $\mathcal U$ from $z_1$ to $z_2$ \emph{not} contained in $\mathcal B$ (i.e., $\mathcal U\setminus\mathcal P$) can only intersect with $\mathcal B$ within the $\ell_\infty$ balls of radius $\frac2{100}$ centered at the points $(-1,1)$ and $(1,1)$. This is because if $\mathcal U\setminus\mathcal P$ intersects with $\mc B$ outside these $\ell_\infty$ balls, then by definition of $\ell^*$, the network cannot make further modifications in $\mc B$, hence contradicting the approximation assumption $\|f^*-f\|_\infty\le\frac1{100}$. 
In other words, all proper intersections between $\mathcal U$ and $\mathcal R$ are \emph{identical} to those between $\mathcal V$ and $\mathcal R$. This implies that $\pi_{\mathcal U}(g_{\ell^*}(q))=1$ and hence, $g_{\ell^*}(q)$ is in the bounded path-connected component of $\mathbb R^2\setminus\mathcal U$.
This completes the proof of the claim \eqref{eq:claim} and therefore, completes the proof of Lemma \ref{lem:topology2}.

\subsection{Proof of Lemma \ref{lem:topologypolygon}}\label{sec:pflem:topologypolygon}
Suppose that $g_{\ell^*}([0,p_1])\cap g_{\ell^*}([p_2,1])=\emptyset$ and $g_{\ell^*}([0,p_1])\setminus\mathcal B$ is contained in a bounded path-connected component of $\mathbb R^2\setminus\mathcal U$ for some $\mathcal U\subset  g_{\ell^*}([p_2,1])\cup\mathcal B$.
If $g_{\ell^*}([0,p_1])\setminus\mathcal B$ is path-connected, then the statement of Lemma \ref{lem:topologypolygon} directly follows.
Hence, suppose that $g_{\ell^*}([0,p_1])\setminus\mathcal B$ has more than one path-connected components.
To help the proof, we introduce the following Lemma.
\begin{lemma}\label{lem:boundary}
If $g_{\ell^*}(p)\in\mathcal\partial\mathcal B$ for some $p\in[0,1]$, then $f(p)=g_{\ell^*}(p)$.
\end{lemma}
\begin{proof}[Proof of Lemma \ref{lem:boundary}]
Suppose that $f(p)\ne g_{\ell^*}(p)$. Then, $\phi_{\ell}^{-1}\circ\sigma\circ\phi_\ell(g_{\ell^*}(p))\ne g_{\ell^*}(p)$ for some $\ell>\ell^*$.
By Lemma \ref{lem:relu}, there exist $a_1,a_2\in\mathbb R^2$ and $b_1,b_2\in\mathbb R$ such that $\phi_{\ell}^{-1}\circ\sigma\circ\phi_\ell(x)= x$ if and only if $\langle a_1,x\rangle+b_1\ge0,\langle a_2,x\rangle+b_2\ge0$. Without loss of generality, we assume that $\langle a_1,g_{\ell^*}(p)\rangle+b_1<0$.
Since $g_{\ell^*}(p)\in\partial\mathcal B$, there exists $z\in\mathcal B$ such that $\langle a_1,z\rangle+b_1<0$, i.e., $\phi_{\ell}^{-1}\circ\sigma\circ\phi_\ell(z)\ne z$, which contradicts to the definition of $\ell^*$ by Lemma \ref{lem:relu}.
This completes the proof of Lemma \ref{lem:boundary}
\end{proof}

By Lemma \ref{lem:boundary} and the assumption that $\|f^*-f\|\le\frac1{100}$, $g_{\ell^*}([0,p_1])\setminus\mathcal B$ can only intersect with $\partial\mathcal B$ within the $\ell_\infty$ ball $\mathcal O$ of radius $\frac2{100}$ centered at the point $(0,1)$. 
Hence, all path-connected components of $g_{\ell^*}([0,p_1])\setminus\mathcal B$ intersect with the line segment $\partial\mathcal B\cap\mathcal O$.
In other words, $g_{\ell^*}([0,p_1])\setminus\mathcal B$ is in a path-connected component of $\mathbb R^2\setminus(g_{\ell^*}([p_2,1])\cup\mathcal B)$ unless $g_{\ell^*}([p_2,1])$ intersects with $\partial\mathcal B \cap\mathcal O$.
However, by Lemma \ref{lem:boundary} and the assumption that $\|f^*-f\|\le\frac1{100}$, $g_{\ell^*}([p_2,1])$ must not intersect with $\partial\mathcal B\cap\mathcal O$.
This completes the proof of Lemma \ref{lem:topologypolygon}.

\end{document}